\documentclass[a4paper]{svproc}

\usepackage{overpic}
\usepackage{color}
\usepackage{soul}
\usepackage{bm}
\usepackage{multirow}
\usepackage{multicol}
\usepackage{graphicx}
\usepackage{float}
\usepackage{subfig}

\usepackage{amsmath}
\usepackage{amsfonts}
\usepackage{amssymb}

\usepackage{amsthm}
\usepackage{bm}
\usepackage{bbm}
\usepackage{mathtools}
\usepackage{thmtools,thm-restate}
\usepackage{algorithm}
\usepackage{algorithmic}
\usepackage{color}
\usepackage{graphicx}
\usepackage{comment}

\def\CC{\mathcal{C}}
\def\EE{\mathcal{E}}
\def\II{\mathcal{I}}
\def\KK{\mathcal{K}}
\def\MM{\mathcal{M}}\def\NN{\mathcal{N}}

\def\TT{\mathcal{T}}

\def\Bb{\mathbf{B}}

\def\Gb{\mathbf{G}}

\def\Mb{\mathbf{M}}

\def\ab{\mathbf{a}}
\def\fb{\mathbf{f}}
\def\gb{\mathbf{g}}

\def\ub{\mathbf{u}}
\def\xb{\mathbf{x}}

\def\Rbb{\mathbb{R}}

\def\ett{\mathtt{e}}

\def\ltt{\mathtt{l}}

\def\rtt{\mathtt{r}}
\def\utt{\mathtt{u}}
\def\vtt{\mathtt{v}}

\def\R{\Rbb}

\def\t{\top}
\def\*{\star}

\usepackage[dvipsnames]{xcolor}

\newcommand{\q}{\mathbf{q}}
\newcommand{\qd}{{\dot{\q}}}
\newcommand{\qdd}{{\ddot{\q}}}

\newcommand{\x}{\mathbf{x}}
\newcommand{\xd}{{\dot{\x}}}
\newcommand{\xdd}{{\ddot{\x}}}
\newcommand{\y}{\mathbf{y}}
\newcommand{\yd}{{\dot{\y}}}

\newcommand{\z}{\mathbf{z}}
\newcommand{\zd}{{\dot{\z}}}
\newcommand{\zdd}{{\ddot{\z}}}

\newcommand{\f}{\mathbf{f}}

\newcommand{\g}{\mathbf{g}}

\newcommand{\J}{\mathbf{J}}
\newcommand{\Jd}{{\dot{\J}}}

\newcommand{\B}{\mathbf{B}}

\newcommand{\G}{\mathbf{G}}

\newcommand{\I}{\mathbf{I}}

\newcommand{\M}{\mathbf{M}}

\newcommand{\sdot}[2]{\overset{\lower0.1em\hbox{$\scriptscriptstyle #2$}}{#1}}

\usepackage{url}

\usepackage{xspace}
\def\flow{RMP{flow}\xspace}
\def\algebra{RMP-algebra\xspace}
\def\tree{RMP-tree\xspace}
\def\forest{RMP-forest\xspace}

\def\pushforward{\texttt{pushforward}\xspace}
\def\pullback{\texttt{pullback}\xspace}
\def\resolve{\texttt{resolve}\xspace}

\newenvironment{acknowledgment}{%
      \list{}{}\item[\hskip\labelsep\bfseries Acknowledgments]}
    {\endlist}

\begin{document}

\mainmatter              

\title{Multi-Objective Policy Generation \\ for
Multi-Robot Systems \\ Using
Riemannian Motion Policies}

\titlerunning{Multi-Objective Policy Generation for Multi-Robot Systems Using RMPs}


\author{
Anqi~Li
\and
Mustafa~Mukadam
\and
Magnus~Egerstedt
\and
Byron~Boots
}

\institute{
    Georgia Institute of Technology,
    Atlanta, Georgia 30332, USA\\
    \email{\{anqi.li, mhmukadam, magnus\}@gatech.edu, bboots@cc.gatech.edu}
}

\authorrunning{Li, Mukadam, Egerstedt, Boots}

\maketitle

\vspace{-4mm}
\begin{abstract}
In many applications, multi-robot systems are required to achieve multiple objectives. For these multi-objective tasks, it is oftentimes hard to design a single control policy that fulfills all the objectives simultaneously.
In this paper, we focus on  multi-objective tasks that can be decomposed into a set of simple subtasks. Controllers for these subtasks are individually-designed and then combined into a control policy for the entire team.
One significant feature of our work is that the subtask controllers are designed along with their underlying manifolds. When a controller is combined with other controllers, their associated manifolds are also taken into account. This formulation yields a policy generation framework for multi-robot systems that can combine controllers for a variety of objectives while implicitly handling the interaction among robots and subtasks. To describe controllers on manifolds, we adopt Riemannian Motion Policies (RMPs), and propose a collection of RMPs for common multi-robot subtasks. Centralized and decentralized algorithms are designed to combine these RMPs into a final control policy. 
Theoretical analysis shows that the system under the control policy is stable. Moreover, we prove that many existing multi-robot controllers can be closely approximated by the framework. The proposed algorithms are validated through both simulated tasks and robotic implementations.

\end{abstract}

\vspace{-8mm}
\section{Introduction}
\vspace{-2mm}

Multi-robot control policies are often designed through performing gradient descent on a potential function that encodes a \emph{single} team-level objective, e.g. forming a certain shape, covering an area of interest, or meeting at a common location~\cite{bullo2009distributed,mesbahi2010graph,cortes2017coordinated}. However, many problems involve a diverse set of objectives that the robotic team needs to fulfill simultaneously. For example, collision avoidance and connectivity maintenance are often required in addition to any primary tasks~\cite{wang2016multi}. One possible solution is to encode the multi-objective problem as a single motion planning problem with various constraints~\cite{desaraju2012decentralized,wagner2015subdimensional,luo2016distributed,swaminathan2015planning}.
However, as more objectives and robots are considered, it can be difficult to directly search for a solution that can achieve \emph{all} of the objectives simultaneously. 

An alternative strategy is to design a controller for each individual objective and then \emph{combine} these controllers into a single control policy. Different schemes for combining controllers have been investigated in the multi-robot systems literature. For example, one standard treatment for inter-robot collision avoidance is to let the collision avoidance controller take over the operation if there is a risk of collision~\cite{arkin1998behavior}. 
A fundamental challenge for such construction is that unexpected interaction between individual controllers can yield the overall system unstable~\cite{ames2014control,wang2017safety}.  Null-space-based behavioral control~\cite{bishop2003use,antonelli2010flocking} 
forces low priority controllers to not interfere with high priority controllers. However, when there are a large number of objectives, the system may not have the sufficient degrees of freedom to consider all the objectives simultaneously.
Another example is the potential field method, which formulates the overall controller as a weighted sum of controllers for each objective~\cite{arkin1998behavior,KhatibPotentialFields1985}. While the system is guaranteed to be stable when the weights are constant, careful tuning of these constant weights are required to produce desirable behaviors.

Methods based on Control Lyapunov Functions (CLFs) and Control Barrier Functions (CBFs)~\cite{wang2016multi,ames2014control,wang2017safety} seek to optimize primary task-level objectives while
formulating secondary objectives, such as collision avoidance and connectivity maintenance, as CLF or CBF constraints and solve via quadratic programming (QP). While this provides a computational framework for general multi-objective multi-robot tasks, solving the QP often requires centralized computation and can be computationally demanding if the number of robots or constraints is large. Although the decentralized safety barrier certificate~\cite{wang2017safety} is a notable exception, it only considers inter-robot collision avoidance and it has not been demonstrated how the same decentralized construction can be applicable to other objectives.

In this paper, we return to the idea of combining controllers and rethink how an objective and its corresponding controller are defined: instead of defining objectives directly on the configuration space,
we define them on non-Euclidean \emph{manifolds}, which can be lower-dimensional than the configuration space. When combining individually-designed controllers, we consider the outputs of the controllers and their underlying manifolds. In particular, we adopt Riemannian Motion Policies (RMPs)~\cite{ratliff2018riemannian}, a class of manifold-oriented control policies that has been successfully applied to robot manipulators, and \flow~\cite{cheng2018rmpflow}, the computational framework for combining RMPs. This framework, where each controller is associated with a matrix-value and state-dependent weight, can be considered as an extension to the potential field method. This extension leads to new geometric insight on designing controllers and more freedom to combine them.
While the \flow algorithm is centralized, we provide a decentralized version and establish the stability analysis for the decentralized framework.

There are several major advantages to defining objectives and controllers on manifolds for multi-robot systems.
First, this formulation provides a general formula for the construction of controllers: the key step is to design the manifold for each substask, as controllers/desired behaviors can be viewed as a natural outcome of their associated manifolds. For example, obstacle avoidance behavior is closely related to the geodesic flow in a manifold where the obstacles manifest as holes in the space. Second, since we design controllers in the manifolds most relevant to their objectives, these manifolds are usually of lower dimension than the configuration space. When properly combined with other controllers, this can provide additional degrees of freedom that help controllers avoid unnecessary conflicts. This is particularly important for multi-robot systems where a large number of controllers interact with one another in a complicated way.
Third, it is shown in~\cite{cheng2018rmpflow} that Riemannian metrics on manifolds naturally provide a notion of importance that enables the stable combination of controllers.
Finally, \flow is coordinate-free~\cite{cheng2018rmpflow}, which allows the proposed framework to be directly generalized to heterogeneous multi-robot teams.

We present four contributions in this paper.
First, we present a centralized solution to combine controllers for solving multi-robot tasks based on \flow.
Second, we design a collection of RMPs for simple and common multi-robot subtasks that can be combined to achieve more complicated tasks. Third, we draw a connection between some of the proposed RMPs and a large group of existing multi-robot distributed controllers.
Finally, we introduce a decentralized extension to \flow, with its application to multi-robot systems,
and establish the stability analysis for this decentralized framework.


\vspace{-2mm}
\section{Riemannian Motion Policies (RMPs)}\label{sec:rmp}
\vspace{-2mm}

We briefly review Riemannian Motion Policies (RMPs)~\cite{ratliff2018riemannian}, a mathematical representation of policies on manifolds, and \flow~\cite{cheng2018rmpflow}, a recursive algorithm to combine RMPs. In this section, we start with RMPs and \flow for single robots, for which these concepts are initially defined~\cite{ratliff2018riemannian,cheng2018rmpflow}. We will later on consider them in the context of multi-robot systems in subsequent sections.

Consider a robot (or a group of robots in later sections) with its configuration space $\CC$ being a smooth $d$-dimensional manifold. For the sake of simplicity, we assume that $\CC$ admits a global\footnote{In the case when $\CC$ does not admit a global coordinate, a similar construction can be done locally on a subset of the configuration space $\CC$.
} \emph{generalized coordinate} $\q: \CC \to \R^d$. As is the case in \cite{cheng2018rmpflow}, we assume that the system can be feedback linearized in such a way that it is controlled directly through the generalized acceleration, $\qdd = \ub(\q, \qd)$. We call $\ub$ a \textit{policy} or a \textit{controller}, and
$(\q, \qd)$  the \emph{state}.

\flow~\cite{cheng2018rmpflow} assumes that the task is composed of a set of \emph{subtasks}, for example, avoiding collision with an obstacle, reaching a goal, tracking a trajectory, etc. In this case, the task space, denoted $\TT$, becomes a collection of multiple \emph{subtask spaces}, each corresponding to a subtask.
We assume that the task space $\TT$ is related to the configuration space $\CC$ through a smooth \emph{task map} $\psi: \CC \to \TT$.
The goal of \flow~\cite{cheng2018rmpflow} is to generate policy $\ub$ in the configuration space $\CC$ so that the trajectory exhibits desired behaviors on the task space $\TT$.

\vspace{-3mm}
\subsection{Riemannian Motion Policies}

Riemannian Motion Policies (RMPs)~\cite{ratliff2018riemannian} represent policies on manifolds. Consider an $m$-dimensional manifold $\MM$ with generalized coordinate $\x \in \R^m$. 
An RMP on $\MM$ can be represented by two forms, its \emph{canonical form} and its \emph{natural form}. The \emph{canonical form} of an RMP is a pair $(\ab, \M)^\MM$, where 
$\ab: (\x,\xd)\mapsto \ab(\x,\xd)\in\R^m$ is the desired acceleration, i.e. control input, and 
$\M: (\x,\xd)\mapsto \M(\x,\xd)\in\R_+^{m\times m}$ is the inertial matrix which defines the importance of the RMP when combined with other RMPs. Given its canonical form, the \emph{natural form} of an RMP is the pair $[\f, \M]^\MM$, where $\f = \M\,\ab$ is the desired force. The natural forms of RMPs are introduced mainly for computational convenience.

RMPs on a manifold $\MM$ can be naturally (but not necessarily) generated from a class of systems called \emph{Geometric Dynamical Systems}~(GDSs)~\cite{cheng2018rmpflow}. GDSs are a generalization of the widely studied classical \emph{Simple Mechanical Systems} (SMSs)~\cite{bullo2004geometric}. In GDSs, the kinetic energy metric, $\Gb$, is a function of both the configuration and velocity, 
i.e. $\Gb(\x,\xd)\in\R^{m\times m}_{+}$. This allows kinetic energy to be dependent on the direction of motion, which can be useful in applications such as obstacle avoidance~\cite{ratliff2018riemannian,cheng2018rmpflow}. The dynamics of GDSs are in the form of \vspace{-1mm}
\begin{equation}\label{eq:GDS}
 \left(\Gb(\x,\xd) + \bm\Xi_{\G}(\x,\xd)\right)\,\xdd
+ \bm\xi_{\G}(\x,\xd)  = - \nabla_\x \Phi(\x) - \Bb(\x,\xd)\,\xd,
\end{equation}
where we call $\B: \R^m \times \R^m \to \R^{m\times m}_{+}$ the \emph{damping matrix} and $\Phi: \R^m \to \R$ the \emph{potential function}. \mbox{The curvature terms $\bm\Xi_{\Gb}$ and $\bm\xi_{\Gb}$ are induced by metric $\Gb$, }\vspace{-1mm}
\begin{equation}\label{eq:curvatures}
    \begin{split}
    \bm\Xi_{\G}(\x,\xd)&\,\coloneqq\,\frac{1}{2} \sum_{i=1}^m \, \dot{x}_i\,\partial_{\xd}\, \gb_{i}(\x,\xd),\\
    \bm\xi_{\G}(\x,\xd)&\,\coloneqq\,\sdot{\Gb}{\x}(\x,\xd)\,\xd - \frac{1}{2} \nabla_\x\, (\xd^\t \Gb(\x,\xd)\, \xd),
    \end{split}
\end{equation}
with $\sdot{\Gb}{\xb}(\x,\xd) \coloneqq  [\partial_{\x} \, \gb_{i} (\x,\xd)\, \xd]_{i=1}^m$, $\gb_{i}$ denoting the $i$th column of $\Gb$, $x_i$ denoting the $i$th component of $\x$, and $[\cdot]$ denoting matrix composition through horizontal concatenation of vectors. 
Given a GDS~\eqref{eq:GDS}, there is an RMP $(\ab, \M)^\MM$ naturally associated with it given by $\ab=\xdd$ and $\M(\x,\xd)=\left(\Gb(\x,\xd) + \bm\Xi_{\G}(\x,\xd)\right)$. Therefore, the velocity dependent metric $\Gb$ provides velocity dependent importance weight $\Mb$ when combined with other RMPs.

\vspace{-3mm}
\subsection{\flow}

\flow~\cite{cheng2018rmpflow} is an algorithm to generate control policies on the configuration space given the RMPs for all subtasks, for example, collision avoidance with a particular obstacle, reaching a goal, etc. Given the state information of the robot in the configuration space and a set of individually-designed controllers (RMPs) for the subtasks, \flow produces the control input on the configuration space through combining these controllers.

\flow introduces: i) a data structure, the \emph{\tree}, to describe the structure of the task map $\psi$, and ii) a set of operators, the \emph{\algebra}, to propagate information across the \tree. An \tree is a directed tree. Each node $\utt$ in the \tree is associated with a state $(\x,\xd)$ defined over a manifold ${\MM}$ together with an RMP $(\fb_{\utt}, \M_{\utt})^{\MM}$.
Each edge $\ett$ in the \tree is augmented with a smooth map from the parent node to the child node, denoted as $\psi_{\ett}$. An example \tree is shown in Fig.~\ref{fig:tree}. The root node of the \tree $\rtt$ is associated with the state of the robot $(\q,\qd)$ and its control policy on the configuration space $(\fb_\rtt,\M_\rtt)^\CC$. Each leaf node $\ltt_k$ corresponds to a subtask with its control policy given by an RMP $(\fb_{\ltt_k},\M_{\ltt_k})^{\TT_{k}}$, where $\TT_{k}$ is a subtask space. 

\begin{figure}
    \centering
    \vspace{-2mm}
    \resizebox{!}{1.3in}{\includegraphics{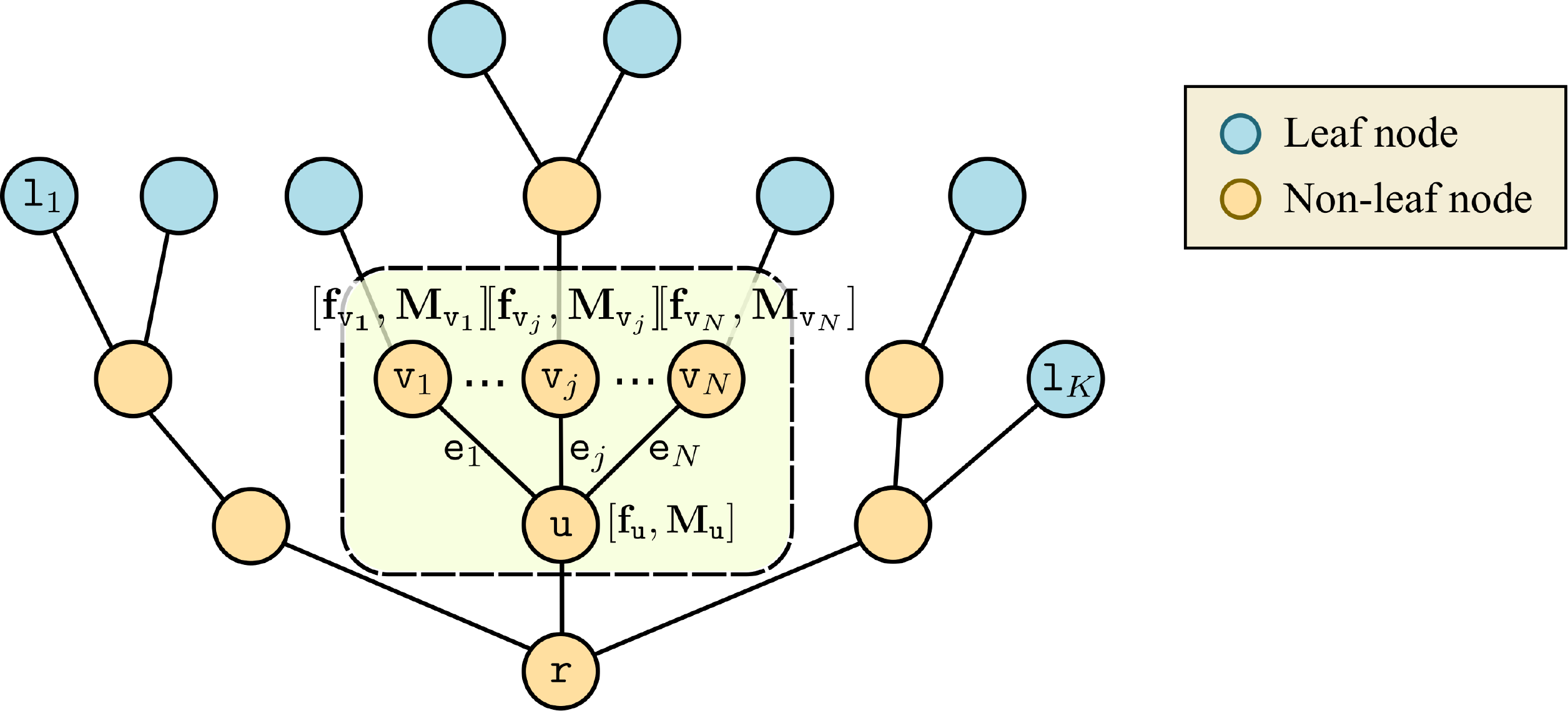}}
    \vspace{-2mm}
    \caption{An example of an \tree. See text for details.
    }
    \label{fig:tree}
    \vspace{-6mm}
\end{figure}

The \algebra consists of three operators: \pushforward, \pullback and \resolve. To illustrate how they operate, consider a node $\utt$ with $N$ child nodes, denoted as $\{\vtt_j \}_{j=1}^N$. Let $\{\ett_j \}_{j=1}^N$ be the edges from $\utt$ to the child nodes (Fig.~\ref{fig:tree}). Suppose that $\utt$ is associated with the manifold $\MM$, while each child node $\vtt_j$ is associated with the manifold $\NN_j$.
The \algebra works as follows: \vspace{-1mm}
\begin{enumerate}
\item The \pushforward operator forward propagates the \emph{state} from the parent node $\utt$ to its child nodes $\{\vtt_j \}_{j=1}^N$. Given the state $(\x,\xd)$ associated with $\utt$, the \pushforward operator at $\vtt_j$ computes its associated state as $(\y_j, \yd_j) = (\psi_{\ett_j}(\x) , \J_{\ett_j} (\x)\,\xd )$, where $\J_{\ett_j} = \partial_\x \psi_{\ett_j}$ is the Jacobian matrix of the map $\psi_{\ett_j}$.

\item The \pullback operator combines the RMPs from the child nodes $\{\vtt_j \}_{j=1}^N$ to obtain the RMP associated with the parent node $\utt$. Given the RMPs from the child nodes, $\{[\f_{\vtt_j}, \M_{\vtt_j} ]^{\NN_j}\}_{j=1}^N$, the RMP associated with node $\utt$, $[\f_\utt,\M_\utt]^{\MM}$, is computed by the \pullback operator as, \vspace{-2mm}
\begin{equation*}\label{eq:natural-pullback} 
\f_\utt = \sum_{j=1}^N \J_{\ett_j}^\t (\f_{\vtt_j} - \M_{\vtt_j} \Jd_{\ett_j} \xd),\qquad \M_\utt = \sum_{j=1}^N \J_{\ett_j}^\t \M_{\vtt_j} \J_{\ett_j}.\vspace{-2mm}
\end{equation*}


\item The \resolve operator maps a natural-formed RMP to its canonical form. Given the natural-formed RMP $[\f_\utt, \M_\utt]^{\MM}$, the operator produces $(\ab_\utt, \M_\utt)^{\MM}$ with $\ab_\utt = \M^{\dagger}\,\f_\utt$, where $\dagger$ denotes Moore-Penrose inverse.
\end{enumerate}

With the \tree specified, \flow can perform control policy generation through the following process.
First, \flow performs a forward pass: it recursively calls \pushforward from the root node to the leaf nodes to update the state information associated with each node in the \tree. Second, every leaf node $\ltt_k$ \emph{evaluates} its corresponding natural-formed RMP $\{[\fb_{\ltt_k},\M_{\ltt_k}]^{\TT_{k}}\}$, possibly given by a GDS. Next, \flow performs a backward pass: it recursively calls \pullback from the leaf nodes to the root node to back propagate the RMPs in the natural form. After that, \flow calls \resolve at the root node to transform the RMP $[\f_{\rtt}, \M_{\rtt}]^\CC$ into its canonical form $(\ab_{\rtt}, \M_{\rtt})^\CC$. Finally, the robot executes the control policy by setting $\qdd = \ub=\ab_{\rtt}$.

\vspace{-2mm}
\subsection{Stability Properties of \flow}



To establish the stability results of \flow, we assume that every leaf node is associated with a GDS. Before stating the stability theorem, we need to define the metric, damping matrix, and potential function for a \emph{node} in the \tree.
\begin{definition}
If a node is a leaf, its metric, damping matrix and potential function are defined as in its associated GDS~\eqref{eq:GDS}. Otherwise, 
let $\{\vtt_j\}_{j=1}^N$ and $\{\ett_j\}_{j=1}^N$ denote the set of all child nodes of $\utt$ and associated edges, respectively. Suppose that $G_{\vtt_j}$, $B_{\vtt_j}$ and $\Phi_{\vtt_j}$ are the metric, damping matrix, and potential function for the child node $\vtt_j$. Then, the metric $\Gb_\utt$, damping matrix $\Bb_\utt$ and potential function $\Phi_\utt$ for the node $\utt$ are defined as,\vspace{-2mm}
\begin{equation}\label{eq:mdp}
        \Gb_\utt = \sum_{j=1}^N\J_{\ett_j}^\t\G_{\vtt_j}\J_{\ett_j},\quad
        \B_\utt = \sum_{j=1}^N\J_{\ett_j}^\t \B_{\vtt_j} \J_{\ett_j},\quad
        \Phi_\utt =  \sum_{j=1}^N\Phi_{\vtt_j} \circ \psi_{\ett_j},\vspace{-2mm}
\end{equation}
where $\circ$ denotes function composition.
\end{definition}

The stability results of \flow are stated in the following theorem.\vspace{-0mm}
\begin{theorem}[Cheng \emph{et al.} \cite{cheng2018rmpflow}]\label{thm:stability}
Let $\G_{\rtt}$, $\B_{\rtt}$, and $\Phi_{\rtt}$ be the metric, damping matrix, and potential function of the root node defined in (\ref{eq:mdp}). If $\G_{\rtt}, \B_{\rtt} \succ 0 $,  and $\M_{\rtt}=(\G_{\rtt} + \bm\Xi_{\Gb_\rtt})$ is nonsingular, the system converges to a forward invariant set $\CC_\infty \coloneqq \{(\q,\qd) : \nabla_\q \Phi_{\rtt} = 0, \qd = 0 \}$.
\end{theorem}

\vspace{-4mm}
\section{Centralized Control Policy Generation}\label{sec:centralized}
\vspace{-2mm}

We begin by formulating a control policy generation algorithm for multi-robot systems directly based on \flow. This algorithm is \emph{centralized} because it requires a centralized processor to collect the states of \emph{all} robots and solve for the control input for \emph{all} robots jointly given \emph{all} the subtasks. In Section \ref{sec:decentralized}, we introduce a decentralized algorithm and analyze its stability properties.

Consider a potentially heterogeneous\footnote{Please see Appendix~\ref{app:heterogeneity} for a discussion about heterogeneous teams.} team of $N$ robots indexed by $\II=\{1,\ldots,N\}$. Let $\CC_i$ be the configuration space of robot $i$ with $\q_i$ being a generalized coordinate on $\CC_i$. The configuration space is then the product manifold $\CC=\CC_1\times\cdots\times\CC_N$. As in Section \ref{sec:rmp}, we assume that each robot is feedback linearized and we model the control policy for each robot as a second-order differential equation $\qdd_i = \ub_i(\q_i, \qd_i)$. An obvious example is a team of mobile robots with double integrator dynamics on $\R^2$. 
Note, however, that the approaches proposed in this paper is not restricted to mobile robots in Euclidean spaces.

Let $\KK=\{1,\ldots,K\}$ denote the index set of all subtasks. For each subtask $k\in\KK$,
a controller is individually designed to generate RMPs on the subtask manifold $\TT_{k}$. 
Here we assume that the subtasks are \emph{pre-allocated} in the sense that each subtask $\ltt_k$ is defined for a specified subset of robots $\II_k$. 
Examples of subtasks include collision avoidance between a pair of robots (a binary subtask), trajectory following for a robot (a unitary subtask), etc. 

\emph{The above formulation gives us an alternative view of multi-robot systems with emphasis on their multi-task nature.} Rather than encoding the team-level task as a global potential function (as is commonly done in the multi-robot literature), we decompose the task as \emph{local} subtasks defined for subsets of robots, and design policies for individual subtasks. The main advantage is that as the task becomes more complex, it becomes increasingly difficult to design a single potential function that renders the desired global behavior. However, it is often natural to decompose global tasks into local subtasks, even for complex tasks, since multi-robot tasks can often come from local specifications~\cite{cortes2017coordinated,wang2016multi}. Therefore, this formulation provides a straightforward generalization to multi-objective tasks. Moreover, this subtask formulation allows us to borrow existing controllers designed for single-robot tasks, such as collision avoidance, goal reaching, etc.

Recall from Section \ref{sec:rmp} that \flow operates on an \tree, a tree structure describing the task space. The main objective of this section is thus to construct an \tree for general multi-robot problems. Note that given a set of subtasks, the construction of the \tree is not unique. One way to construct an \tree is to use non-leaf nodes to represent subset of the team:\vspace{-1mm}
\begin{itemize}
    \item The root node corresponds to the joint configuration space $\CC=\CC_1\times\cdots\times\CC_N$ and its corresponding control policy.
    \item Any leaf node $\ltt_k$ is augmented with a user-specified policy represented as an RMP on the subtask manifold $\TT_{k}$.
    \item Every non-leaf node is associated with a product space of the configuration spaces for a subset of the team.
    \item The parent of any leaf RMP $\ltt_k$ is associated with the joint configuration space $\prod_{i\in\II_{k}}\CC_i$, where $\II_{k}$ are the robots that subtask $\ltt_k$ is defined on.
    \item Consider two non-leaf nodes $\utt$ and $\vtt$ such that $\vtt$ is a decedent of $\utt$ in the \tree. Let $\II_{\utt}$ and $\II_{\vtt}$ be the subset of robots corresponds to node $\utt$ and $\vtt$, respectively. Then $\II_{\vtt}\subseteq \II_{\utt}$.\vspace{-1mm}
\end{itemize}

Fig.~\ref{fig:centralized} shows an example \tree for a team of three robots. The robots are tasked with forming a certain shape and reaching a goal while avoiding inter-robot collisions. The root of the \tree is associated with the configuration space for the team, which is the product of the configuration spaces for all three robots. On the second level, the nodes represent subsets of robots which, in this case, are pairs of robots. Several leaf nodes, such as the ones corresponding to collision avoidance and distance preservation, are children of these nodes as they are defined on pairs of robots. One level deeper is the node corresponding to the configuration space for robot $1$. The goal attractor leaf node is a child of it since the goal reaching subtask is assigned only to robot $1$.

\begin{figure}
    \centering
    \vspace{-4mm}
    \resizebox{!}{1.3in}{\includegraphics{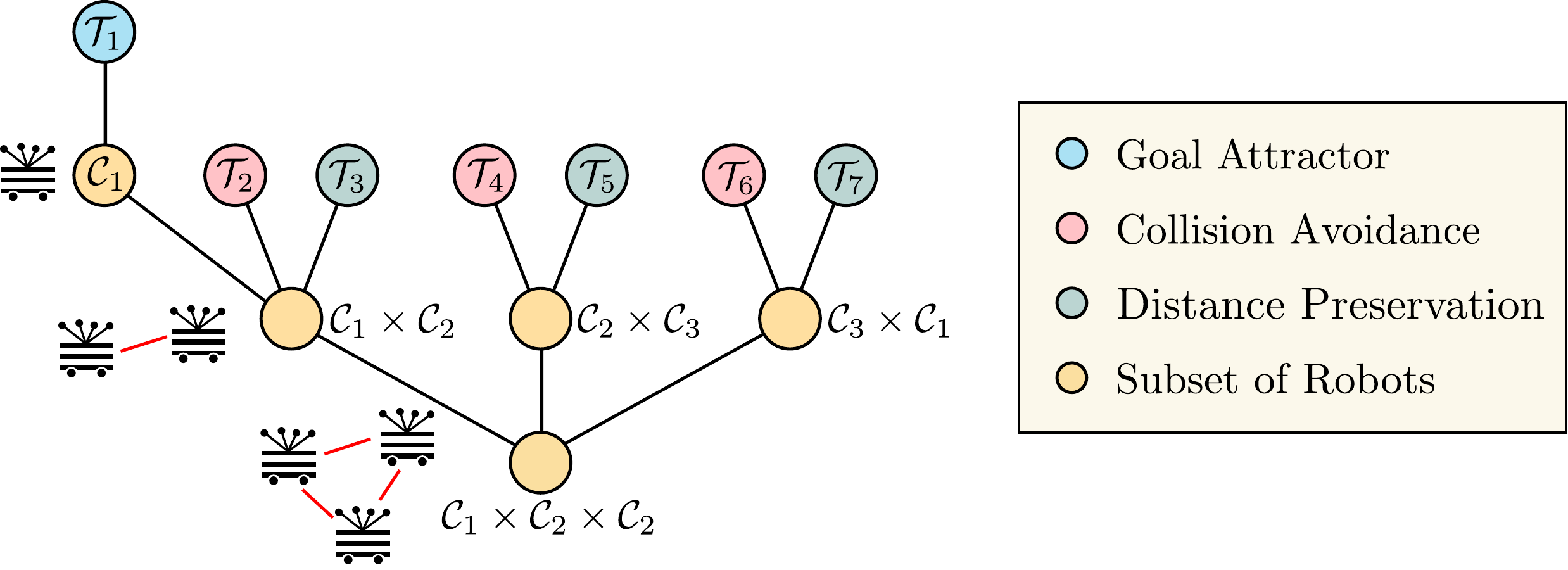}}
    \caption{An example of an \tree for a group of three robots performing a formation preservation task. See text for details.
    }
    \vspace{-4mm}

    \label{fig:centralized}
\end{figure}

Note that branching in the \tree does not necessarily define a partition over robots. Let $\vtt_i$ and $\vtt_j$ be children of the same node $\utt$ and let $\II_{\vtt_i}$ and $\II_{\vtt_j}$ be the subset of robots for node $\vtt_i$ and $\vtt_j$, respectively. Then it is \emph{not} necessary that $\II_{\vtt_i}\,\cap\,\II_{\vtt_j}=\emptyset$. For example, in Fig.~\ref{fig:centralized}, the three nodes on the second level are defined for subsets $\{1, 2\}$, $\{2, 3\}$, and $\{3, 1\}$, respectively. The intersection of any two of them is not empty. In fact, if a branching is indeed a partition, then the problem can be split into \emph{independent} sub-problems. For multi-robot systems, this means that the team consists of independent sub-teams with completely independent tasks. This rarely occurs in practice. 

According to Theorem~\ref{thm:stability}, if all the leaf node controllers are designed through GDSs, it is guaranteed that the controller generated by \flow drives the system to a forward invariant set $\CC_\infty \coloneqq \{(\q,\qd) : \nabla_\q \Phi_\rtt = 0, \qd = 0 \}$
if $\G_\rtt, \B_\rtt \succ 0$ and $\M_\rtt$ is non-singular. In other words, this guarantees that the resulting system is stable, which is important: unstable behaviors such as  high-frequency oscillation are avoided 
and, more importantly, stability provides formal guarantees on the performance of certain types of subtasks such as collision avoidance, which is discussed in Section~\ref{sec:collision_rmp}.

To elucidate the process of designing RMPs and to connect to relevant multi-robot tasks, we provide examples of RMPs for multi-robot systems that can produce complex behaviors when combined. In the following examples, we use $\x_i$ to denote the coordinate of robot $i$ in $\R^2$. An additional map can be composed with the given task maps if robots possess different kinematic structures.

\vspace{-3mm}
\subsection{Pairwise Collision Avoidance}\label{sec:collision_rmp}

To ensure safety operation of the robotic team, inter-robot collisions should be avoided. We formulate collision avoidance as ensuring a minimum safety distance $d_S$ for \emph{every} pair of robots. To generate collision-free motions, for any two robots $i,j\in\II$, we construct a collision avoidance leaf node for the pair. The subtask space is the 1-d distance space, i.e. $z=\psi(\x_i, \x_j)=\|\x_i-\x_j\|/d_S-1$. Here, we use $z$ (italic) to denote that it is a scalar on the 1-d space.

To ensure a safety distance between the pair,
we use a construction similar to the collision avoidance RMP for static obstacles in \cite{cheng2018rmpflow}. The metric for the pairwise collision avoidance RMP is defined as $\Gb(z,\dot{z})=w(z)\,u(\dot{z})$, where $w(z)=\frac{1}{z^4}$,  $u(\dot{z})=\epsilon+\min(0,\dot{z})\,\dot{z}$ with a small positive scalar $\epsilon>0$. The metric retains a large value when the robots are close to each other ($z$ is small), and when the robots are moving fast towards each other ($\dot{z}<0$ and $|\dot{z}|$ is large). Conversely, the metric decreases rapidly as $z$ increases. Recall that the metric is closely related to the inertial matrix, which determines the importance of the RMP when combined with other policies. This means that the collision avoidance RMP dominates when robots are close to each other or moving fast towards each other, while it has almost no effect 
when the robots are far from each other.

We next design the GDS that generates the collision avoidance RMP. The potential function is defined as $\Phi(z)=\frac{1}{2}\alpha\,w(z)^2$ and the damping matrix is defined as $\Bb(z,\dot{z})=\eta \,\G(z,\dot{z})$, where $\alpha,\eta$ are positive scalars. As the robots approach the safety distance, the potential function $\Phi(z)$ approaches infinity. Due to the stability guarantee of \flow, this barrier-type potential will always ensure that the distance between robots is greater than $d_S$. 
This means that the resulting control policy from \flow is \emph{always} collision-free.

\vspace{-2mm}
\subsection{Pairwise Distance Preservation}\label{sec:formation_rmp}

Another common task for multi-robot systems is to form a specified shape or formation. This can be accomplished by maintaining the inter-robot distances between certain pairs of robots. Therefore, formation control can be induced by a set of leaf nodes that maintain distances. Such an RMP can be defined on the 1-d distance space, $z=\psi(\x_i, \x_j)=\|\x_i-\x_j\|-d_{ij}$,
where $d_{ij}$ is the desired distance between robot $i$ and robot $j$. For the GDS, we use a constant metric $\G\,\equiv\,c\in\R_{++}$. The potential function is defined as $\Phi(z)=\frac{1}{2}\,\alpha\,z^2$ and the damping is $\B(\x,\xd)\equiv\eta$, with $\alpha,\eta>0$. We will refer to this RMP as \emph{Distance Preservation RMPa} in later sections.

Note that the above RMP is not equivalent to the potential-based formation controller in, e.g. \cite{mesbahi2010graph,cortes2017coordinated}. However, there does exist an RMP that has very similar behavior. 
It is defined on the product space, $\z=(\x_i,\x_j)$. The metric for the RMP is also constant, $\G\,\equiv\,c\,\I$, where $c\in\R_{++}$ and $\I$ denotes the identity matrix. The potential function is defined as $\Phi(\z)=\frac{1}{2}\,\EE_{ij}(\|\x_i-\x_j\|)$, where $\EE_{ij}:\R\to\R$ is differentiable and achieves its minimum at $d_{ij}$. Common choices include $\EE_{ij}(s)=(s-d_{ij})^2$ and $\EE_{ij}(s)=(s^2-d_{ij}^2)^2$~\cite{mesbahi2010graph}. The damping matrix is defined as $\B\,\equiv\,\eta\, \I$, with $\eta>0$. This RMP will be referred to as \emph{Distance Preservation RMPb} in later sections.

When there are only distance preserving RMPs in the \tree, the resulting individual-level dynamics are given by \vspace{-2mm}
\begin{equation}\label{eq:formation_b}
\xdd_i = -\frac{\alpha}{c\,D_i} \sum_{j:(i,j)\in E}\nabla_{\x_i}\EE_{ij}(\|\x_i-\x_j\|) - \frac{\eta}{c}\,\xd_i, \vspace{-2mm}
\end{equation}
where $E$ represents the set of edges in the formation graph, and $D_i=|\{j:(i,j)\in E\}|$ is the degree of robot $i$. This is closely related to the gradient descent update rule over the potential function $\EE(x)=\frac{1}{2}\,\sum_{(i,j)\in E}\EE_{ij}(\|\x_i-\x_j\|)$ with an additional damping term, and normalized by the degree of the robot. We will later prove in Section \ref{sec:potential_rmp} that the degree-normalized potential-based controller and the original potential-based controller have similar behaviors in the sense that the resulting systems converge to the same invariant set.

The main difference between the two distance preserving RMPs is the space on which they are defined. The first RMP is defined on a 1-d distance space while the second RMP is defined on a higher dimensional space. Therefore, the first RMP is more permissive in the sense that it only specifies desired behaviors in a one dimensional submanifold of the configuration space. 
This is illustrated through a simulated formation preservation task in Section~\ref{sec:results}.

\vspace{-2mm}
\subsection{Potential-based Controllers from RMPs}\label{sec:potential_rmp}

Designing controllers based on the gradient descent rule of a potential function is very common in the multi-robot systems literature, e.g.~\cite{mesbahi2010graph,cortes2017coordinated,cortes2004coverage}. Usually, the overall potential function $\EE$ is the sum of a set of symmetric, pairwise potential functions $\EE_{ij}(\|\x_i-\x_j\|)$ between robot $i$ and robot $j$ that are adjacent in an underlying graph structure. When the robots follow double-integrator dynamics, a damping term is typically introduced to guarantee convergence to an invariant set. Let $\x$ be the ensemble-level state of the team. The controller is given by,
$\xdd = \ub = - \nabla \EE - \eta\,\xd$, where $\eta$ is a positive scalar.
We define a \emph{degree-normalized}  potential-based controller as, $\ub = -\bm\Gamma\,(\nabla \EE + \eta\,\xd)$, where $\bm\Gamma$ is a diagonal matrix with $\bm\Gamma_{ii}=1/D_i$ and $D_i$ is the degree of robot $i$ in the graph. \vspace{-1mm}
\begin{restatable}{theorem}{degreeNormalized}\label{thm:degree_normalized}
Both the degree-normalized controller and the original potential-based controller converge to the invariant set $\{(\x,\xd):\nabla\EE=0, \xd=0\}$. \vspace{-2mm}
\end{restatable}
\begin{proof}
For the original controller, consider the Lyapunov function candidate $V(\x,\xd)=\frac{1}{2}\|\xd\|^2 + \EE(\x)$. Then $\dot{V} = \xd^\t(\xdd + \nabla\EE)=-\eta\,\|\xd\|^2$. By LaSalle's invariance principle~\cite{khalil1996noninear}, the system converges to the set $\{(\x,\xd):\nabla\EE=0, \xd=0\}$. For the degree-normalized controller, consider the Lyapunov function candidate $V(\x,\xd)=\frac{1}{2}\xd^\t \bm\Gamma^{-1}\xd + \EE(\x)$. Then $\dot{V} = \xd^\t(\bm\Gamma^{-1}\xdd + \nabla\EE)=-\eta\,\|\xd\|^2$. The system also converges to the same set by LaSalle's invariance principle~\cite{khalil1996noninear}. \vspace{-1mm}
\end{proof}
Therefore, similar to potential-based formation control, one can directly implement the degree-normalized version
of these potential-based controllers by RMPs defined on the product space, $\z=(\x_i,\x_j)$. The potential function for the RMP is defined as $\Phi(\z)=\EE_{ij}(\|\x_i-\x_j\|)$. Constant metric and damping can be used, e.g. $\G\,\equiv\,c\,\I$, and $\B\,\equiv\,\eta\, \I$, where $c$ and $\eta$ are positive scalars.
Moreover, similar to formation control, one can also define RMPs on the distance space $z=\psi(\x_i, \x_j)=\|\x_i-\x_j\|$ with potential function $\Phi=\EE_{ij}$. Since RMPs are defined on a lower-dimensional manifold, this approach may provide additional degrees of freedom when these RMPs are combined with other policies.

\vspace{-2mm}
\section{Decentralized Control Policy Generation}\label{sec:decentralized}
\vspace{-2mm}

Although the centralized \flow algorithm can be used to generate control policies for multi-robot systems, it can be demanding in both communication and computation. Therefore, 
we develop a decentralized approximation of \flow that 
only relies on local communication and computation.

Before discussing the algorithm, a few definitions and assumptions are needed. 
Given the set of all subtasks $\KK$, we say two robots $i$ and $j$ are neighbors if and only if there exists a subtask such that both robots are involved in. 
We then say that the algorithm is decentralized if only the \emph{state} information of the robot's direct neighbors is required to solve for its control input. Note that here we implicitly assume that the robots are equipped with the sensing modality or communication modality to access the \emph{state} of the neighbors. We also assume that the map and the Jacobian matrix for a subtask are known to the robot if the robot is involved in the subtask. For example, for the formation control task, the robot should know how to calculate distance between two robots given their states, and also know the partial derivatives of the distance function.

The major difference between the decentralized algorithm and the centralized \flow algorithm is that, in the decentralized algorithm, there is no longer a centralized root node that can generate control policies for \emph{all} robots. Instead, each robot should have its \emph{own} \tree that generates policies based on the information available locally. Therefore, the decentralized algorithm actually operates on a \emph{forest} with $N$ RMP-trees, called the \forest. An example 
\forest is shown in Fig.~\ref{fig:forest}. There are three robots performing the same formation preservation task as in Fig.~\ref{fig:centralized}. Hence, there are three RMP-trees in the \forest. For each \tree, there are leaf nodes for every subtask relevant to the robot associated with the \tree. As a result, there are multiple copies of certain subtasks in the \forest, for example, the collision avoidance node for robot $1$ and $2$ appears twice: once in the \tree of robot $1$, and once in the \tree of robot $2$. However, these copies do not share information. 

\begin{figure*}
    \centering
    \vspace{-6mm}
    \resizebox{!}{1.25in}{\includegraphics{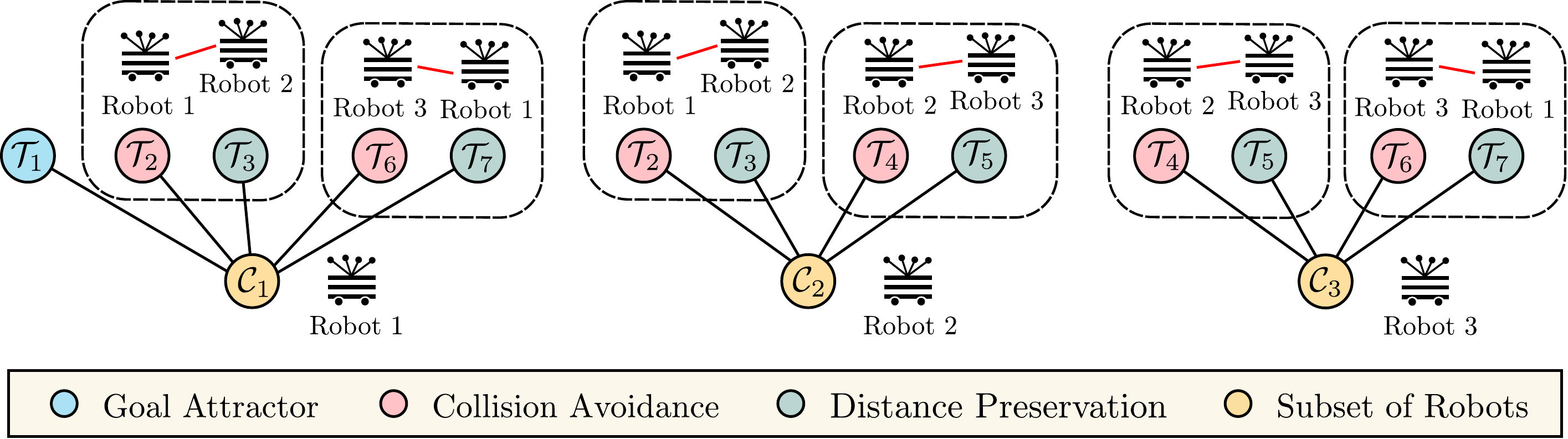}}
    \caption{A decentralized \forest. The three robots are tasked with the same formation preservation task as in Fig.~\ref{fig:centralized}. For the decentralized algorithm, each robot has an individual \tree to solve for its control input. All the leaf RMP nodes that are relevant to the robot are in its \tree. For example, for robot $1$, collision avoidance and distance preservation RMPs for both the pair $\{1, 2\}$ and $\{3, 1\}$ are introduced. There is a goal attractor for robot $1$ since it is the leader. Note that there are several copies of the same subtasks in the forest, 
    however, these copies 
    do not share information with each other. \vspace{-6mm}
    }
    \label{fig:forest}
\end{figure*}

We call the decentralized approximation \emph{partial \flow}. 
In partial \flow, every subtask is viewed as a time-varying unitary task. Therefore, following the \tree construction in the previous section, it is natural to consider one-level RMP-trees, where the leaf nodes are direct children of the root nodes.

Notationwise, let $\KK_i$ be the set of subtasks that robot $i$ participates in. Since there are multiple copies of the same subtasks in the \forest, we use $\ltt_k^i$ to denote the node corresponds to the copy of subtask $k$ in the tree of robot $i$ while let $\ett_k^i$ denote the edge from the root of tree $i$ to the leaf node $\ltt_k^i$. We let $\psi_{\ett_k}$ denote the smooth map from the joint configuration space $\prod_{j\in\II_k}\,\CC_j$ to the subtask space $\TT_k$ (which is the same across trees in the \forest) and let $\J_{\ett_k}^i$ be the Jacobian matrix of $\psi_{\ett_k}$ with respect to $\q_i$ only, i.e. $\J_{\ett_k}^i=\partial_{\q_i}\psi_{\ett_k}$.

To compute the control input for robot $i$, an algorithm similar to \flow is applied in \tree $i$:\vspace{-1mm}
\begin{itemize}
    \item \pushforward: Let $\{\z_{k}^i\}_{k\in\KK_i}$ be the coordinates of the leaf nodes of \tree $i$. Given the state of the root $(\q_i,\qd_i)$, its state is computed as, $\z_{k}^i=\psi_{\ett_k}(\q_{\II_k})$, $\zd_{k}^i=\J_{\ett_k}^i\,\qd_i$, where $\q_{\II_k}=\{\q_j:j\in\II_k\}$. It is worth noting that $\zd_{k}^i\neq \frac{d}{dt} \z_{k}^i$, since the other robots are considered static when computing $\zd_{k}^i$.
    \item \texttt{Evaluate}:
    Let $\Mb_{\ltt_k}$, $\Bb_{\ltt_k}$, and $\Phi_{\ltt_k}$ be the user-designed metric, damping matrix, and potential function for subtask $\ltt_k$. For notational simplicity, we denote $\Gb_{\ltt_k}^i=\Gb_{\ltt_k}(\z_{\ltt_k}^i,\zd_{\ltt_k}^i)$,  $\Bb_{\ltt_k}^i=\Bb_{\ltt_k}(\z_{\ltt_k}^i,\zd_{\ltt_k}^i)$, and $\Phi_{\ltt_k}^i=\Phi_{\ltt_k}(\z_{\ltt_k}^i)$. At leaf node $\ltt_k^i$, the RMP is given by the following system (similar to GDS), \vspace{-3mm}
    \begin{equation}\label{eq:dec_leaf} 
             \fb_{\ltt_k}^i = - \nabla_{\z_k^i} \Phi_{\ltt_k}^i - \Bb_{\ltt_k}^i\,\zd_{k}^i - \frac12\,\sdot{\Gb_{\ltt_k}^i}{\z_{k}^i}(\z_k^i,\zd_{k}^i)\,\zd_{k}^i,\quad
             \Mb_{\ltt_k}^i = \Mb_{\ltt_k}(\y_{\ltt_k}^i,\yd_{\ltt_k}^i)\vspace{-2mm}
    \end{equation}
    Note that, to provide stability, the RMP is no longer generated by a GDS. In particular, the curvature term compensates for the motion of other robots.

    \item \pullback: Given the RMPs from the leaf nodes of tree $i$,  $\{[\f_{\ltt_k}^i,\M_{\ltt_k}^i]^{\TT_k}\}_{k\in\KK_i}$, the \pullback operator calculates the RMP for the root node of tree $i$, \vspace{-1mm}
    \begin{equation}
        \f_{\rtt}^i= \sum_{k\in\KK_i} (\J_{\ett_k}^i)^\t (\f_{\ltt_k}^i - \M_{\ltt_k}^i \dot{\J}_{\ett_k}^i \qd_i),\quad\M_{\rtt}^i = \sum_{k\in\KK_i} (\J_{\ett_k}^i)^\t \M_{\ltt_k}^i\J_{\ett_k}^i.\vspace{-2mm}
    \end{equation}
    \item \resolve: The control input is given by $\ub_i=\ab_{\rtt}^i = (\M_{\rtt}^i)^{\dagger}\,\f_{\rtt}^i$.\vspace{-1mm}
\end{itemize}

Note that when all the metrics are constant diagonal matrices and all the Jacobian matrices are identity matrices, the decentralized partial \flow framework has exactly the same behavior as \flow. This, in particular, holds for the degree-normalized potential-based controllers discussed in Section \ref{sec:potential_rmp}. Therefore, the decentralized partial \flow framework can also reconstruct a large number of multi-robot controllers up to degree normalization.

Partial \flow has a stability result similar to \flow, which is stated in the following theorem. \vspace{-1mm}
\begin{restatable}{theorem}{stabilityDec}\label{thm:stability_dec}
Let $\G_{\rtt}^i = \sum_{k\in\KK_i} (\J_{\ett_k}^i)^\t \G_{\ltt_k}^i\J_{\ett_k}^i$, $\B_{\rtt}^i = \sum_{k\in\KK_i} (\J_{\ett_k}^i)^\t \B_{\ltt_k}^i\J_{\ett_k}^i$, and $\Phi_{\rtt}^i=\sum_{k\in\KK_i}\Phi_{\ltt_k}^i\circ\psi_{\ett_k}$ be the metric, damping matrix, and potential of the tree $i$'s root node. If $\G_{\rtt}^i, \B_{\rtt}^i \succ 0$ and $\M_{\rtt}^i$ is nonsingular for all $i\in\II$, the system converges to a forward invariant set $\CC_\infty \coloneqq \{(\q,\qd) : \nabla_{\q_i} \Phi_{\rtt}^i = 0, \qd_i = 0,\forall i\in\II \}.\vspace{-1mm}$
\end{restatable}

\begin{proof}
See Appendix~\ref{app:proof_stability}.\vspace{-1mm}
\end{proof}

\vspace{-4mm}
\section{Experimental Results}\label{sec:results}
\vspace{-2mm}

We evaluate the multi-robot RMP framework through both simulation and robotic implementation. The detailed choice parameters in the experiments and additional simulation results can be found in Appendix~\ref{app:experiments} and Appendix~\ref{app:simulation}. 


\vspace{-2mm}
\subsection{Simulation Results}

Formation preservation tasks \cite{anderson2008rigid}, where robots must maintain a certain formation while the leader is driven by some external force, are considered harder than formation control tasks since one needs to carefully balance the external force and the formation controller. However, since translations and rotations can still preserve shape, the team \emph{should} have the capability of maintaining the \mbox{formation regardless of the motion of the leader.}

We consider a formation preservation task in simulation where a team of five robots are tasked with forming a regular pentagon while the leader has the additional task of reaching a goal. 
The two distance preservation RMPs introduced in Section \ref{sec:formation_rmp} are compared.  Distance preservation RMPs are defined for all edges in the formation graph. To move the formation, an additional goal attractor RMP is defined for the leader robot, where the construction of the goal attractor RMP can be found in Appendix~\ref{app:attractor_rmp} (referred to as Goal Attractor RMPa) or~\cite{cheng2018rmpflow}. 
We use a damper RMP defined by a GDS on the configuration space of every single robot with only damping  
so that the robots can reach a full stop at the goal. Fig. \ref{fig:fromation_rmp1} shows the resulting behavior for the distance preserving RMPa. 
The robots manage to preserve shape while the leader robot is reaching the goal since the subtasks are defined on lower-dimensional manifolds. By contrast, the behavior for distance preservation RMPb (which is equivalent to the degree-normalized potential-based controller) 
is shown in Fig.~\ref{fig:fromation_rmp2}. This distance preservation RMP fails to maintain the formation when the leader robot is attracted to the goal.

\begin{figure}
    \centering
    \vspace{-8mm}
	\subfloat[Dist. Prsv. RMPa\label{fig:fromation_rmp1}]{
		\resizebox{!}{1.15in}{\includegraphics{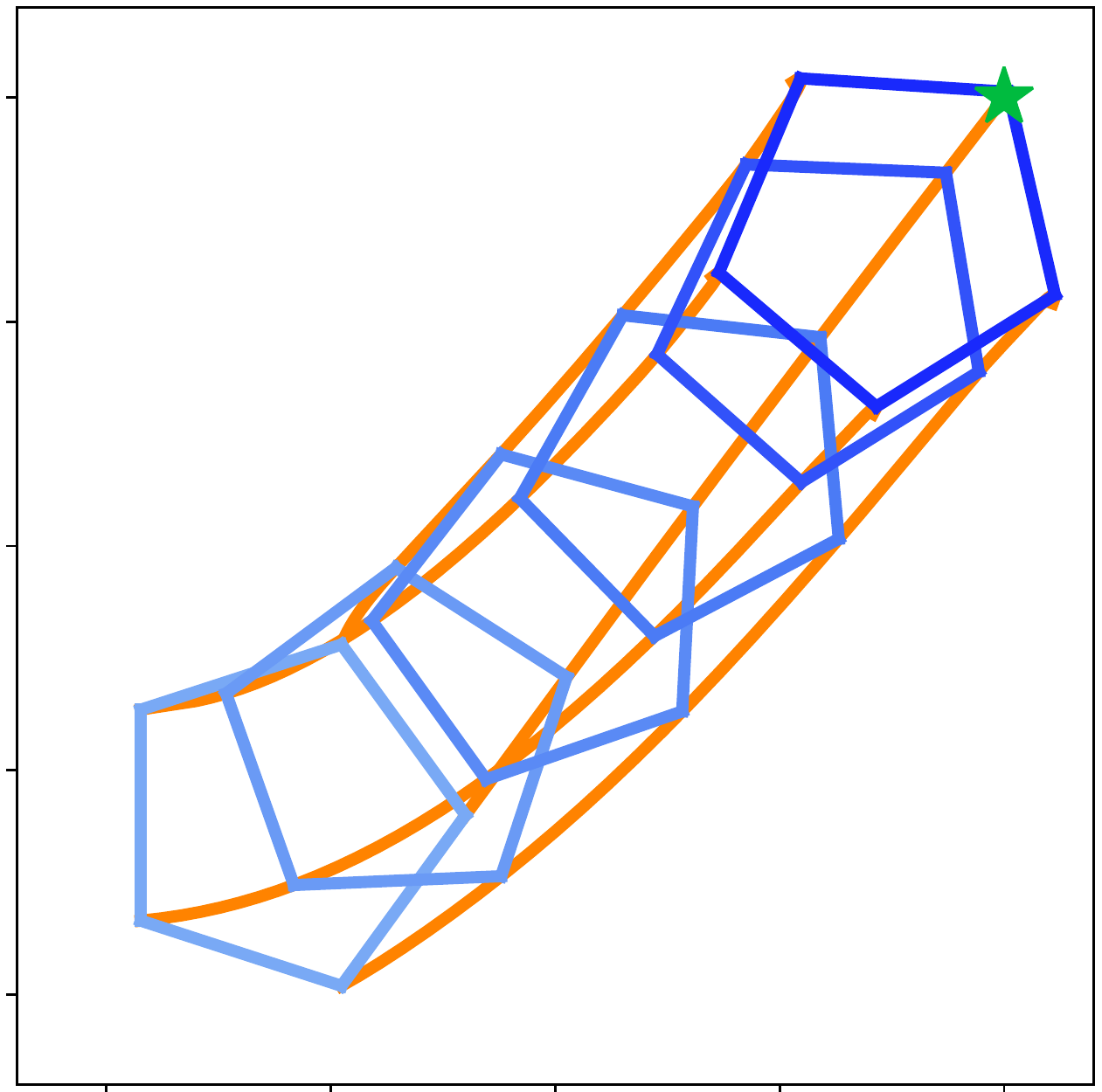}}
	}
	\subfloat[Dist. Prsv. RMPb\label{fig:fromation_rmp2}]{
    	\resizebox{!}{1.15in}{\includegraphics{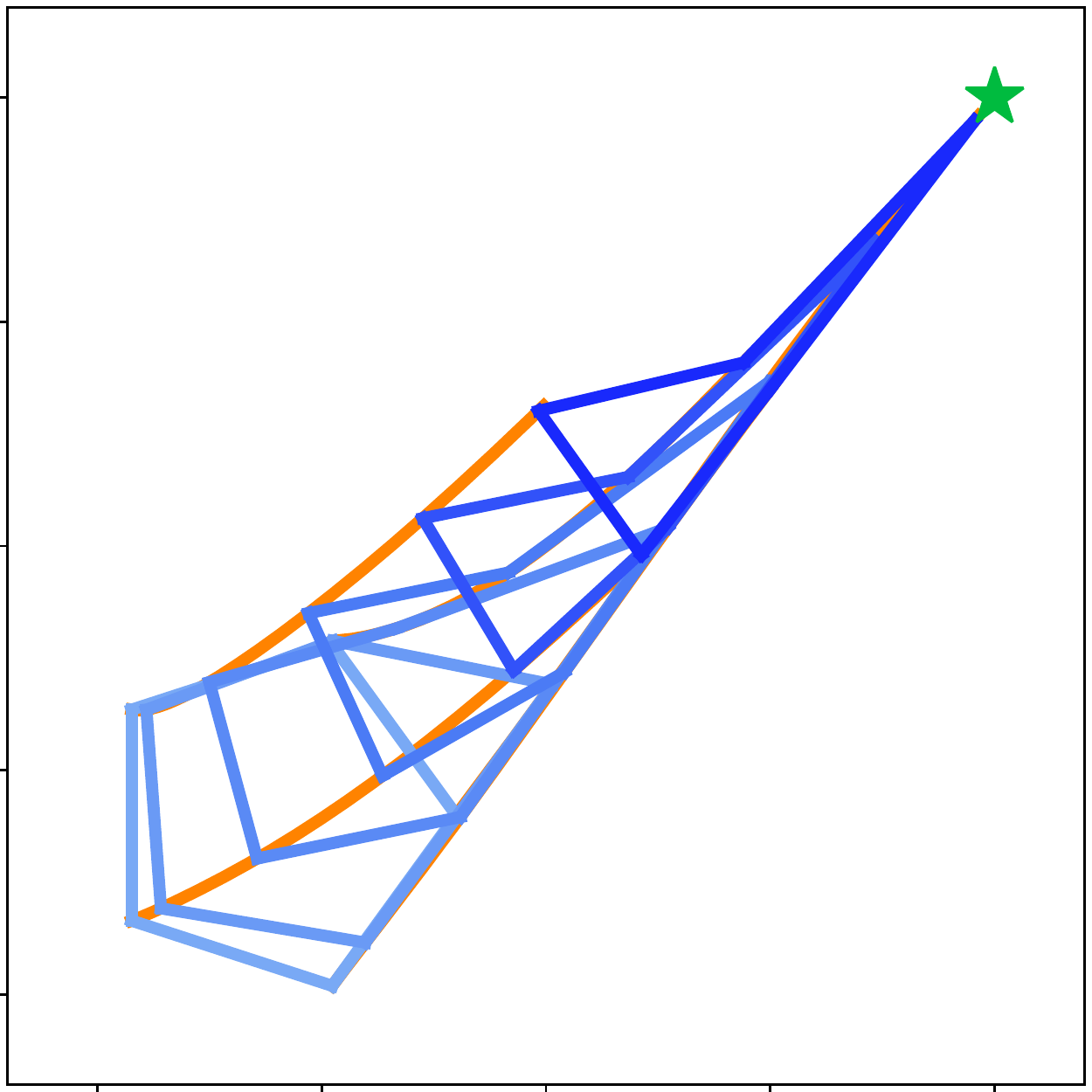}}
	}
	\resizebox{!}{1.15in}{\includegraphics{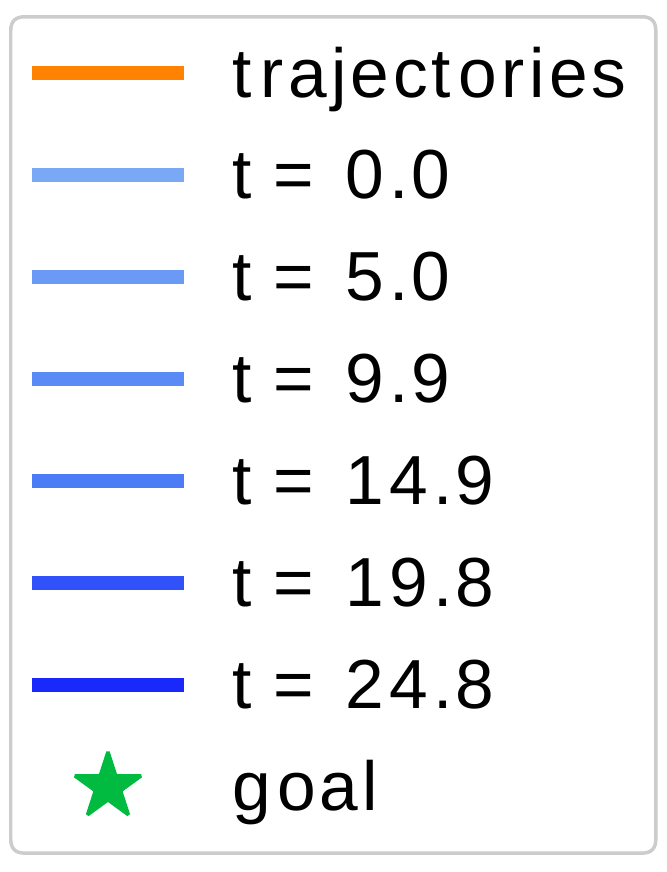}}
	\vspace{-2mm}
    \caption{(a) The behavior of distance preservation RMPa when combined with a goal attractor RMP. The blue pentagons from light to dark denote the shape from $t=0.0s$ to $t=24.8s$. The orange curves show the trajectories of the robots. The robots manage to reach the goal while maintaining the shape. (b) The same task with distance preservation RMPb. The robots fail to maintain the shape.\vspace{-7mm}
    }
    \label{fig:formation_pres}
\end{figure}

\vspace{-4mm}
\subsection{Robotic Implementations}

We present several experiments (video: \url{https://youtu.be/VZHr5SN9wXk})
conducted on the Robotarium~\cite{pickem2017robotarium}, a remotely accessible swarm robotics platform. Since the centralized \flow framework and the decentralized partial \flow frameworks have their own features, we design a separate experiment for each framework to show their full capability.

\vspace{-4mm}
\subsubsection{Centralized \flow Framework}

The main advantage of the centralized \flow framework is that the subtask spaces are jointly considered and hence the behavior of each controller is combined consistently. To fully exploit this feature, we consider formation preservation with two sub-teams of robots. The two sub-teams are tasked with maintaining their formation while moving back and forth between two goal points $A$ and $B$. The five robots in the first sub-team are assigned a regular pentagon formation and the four robots in the second sub-team must form a square. At the beginning of the task, goal $A$ is assigned to the first sub-team and goal $B$ to the second sub-team. 
The robots negotiate their path so that their trajectories are collision free.

A combination of distance preservation RMPs, collision avoidance RMPs, goal attractor RMPs, and damper RMPs are used to achieve this behavior. The construction of the \tree is similar to Fig.~\ref{fig:centralized}. A distance preservation RMPa is assigned to every pair of robots that corresponds to an edge in the formation graph, while collision avoidance RMPs are defined for every pair of robots. For each sub-team, we define a goal attractor RMP for the leader, where the construction of the goal attractor RMP is explained in Appendix~\ref{app:attractor_rmp}. We also use a damper RMP defined by a GDS on the configuration space of every single robot 
so that the robots can reach a full stop at the goal. Fig.~\ref{fig:exp_cen} shows several snapshots from the experiment.  We see that the robots are able to maintain their corresponding formations while avoiding collision. The two sub-teams of robots rotates around each other to avoid potential collision, which shows that the full degrees of freedom of the task is exploited.

\begin{figure*}
    \centering
    \vspace{-7mm}
    \subfloat[$t=0\,\mathrm{s}$\label{fig:exp_cen_0}]{
		\resizebox{!}{0.8in}{\includegraphics{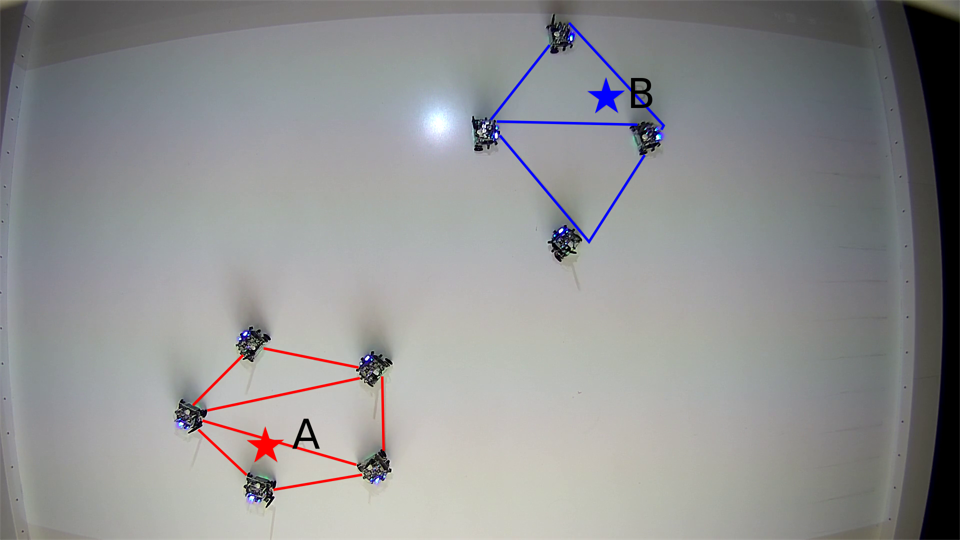}}
	}\quad
	\subfloat[$t=10\,\mathrm{s}$\label{fig:exp_cen_10}]{
    	\resizebox{!}{0.8in}{\includegraphics{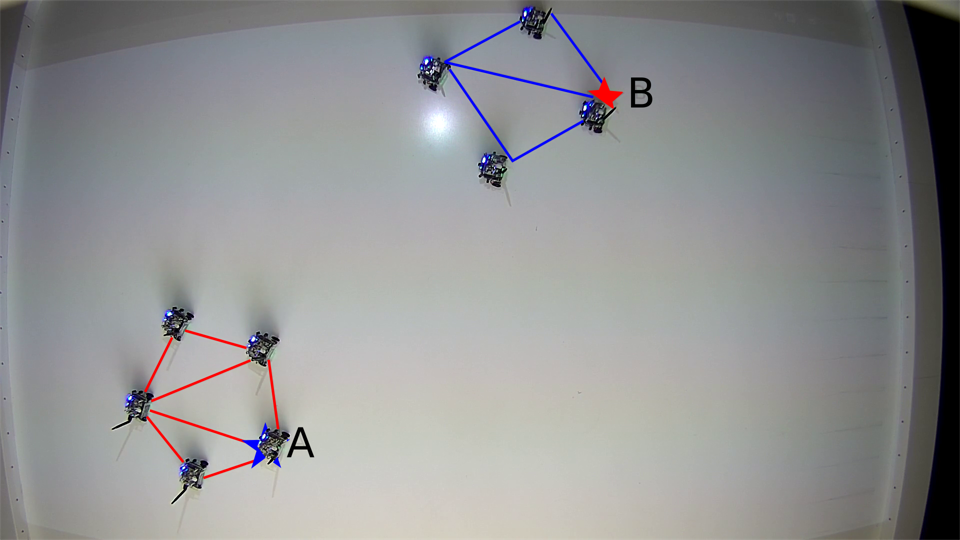}}
	}\quad
	\subfloat[$t=44\,\mathrm{s}$\label{fig:exp_cen_44}]{
		\resizebox{!}{0.8in}{\includegraphics{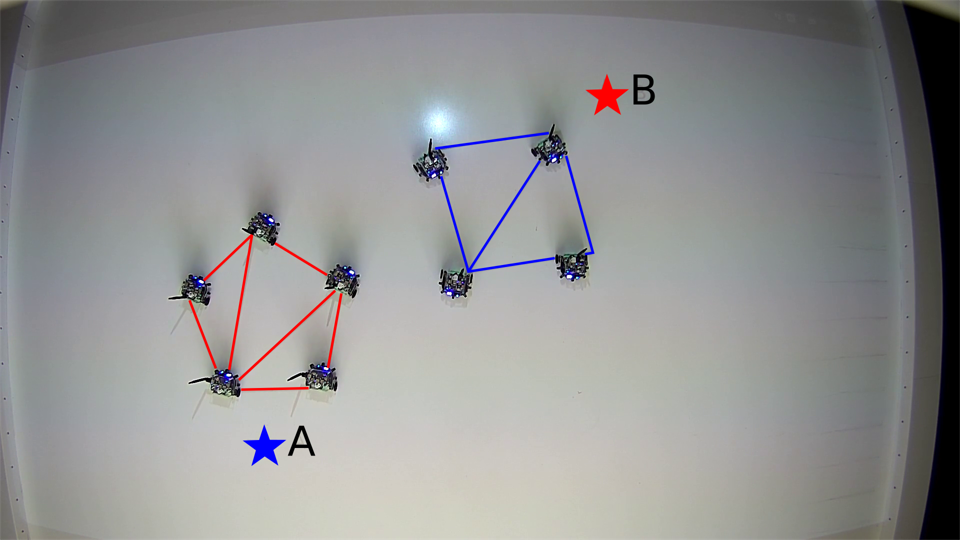}}
	}\vspace{-2mm}
	\\
	\subfloat[$t=64\,\mathrm{s}$\label{fig:exp_cen_64}]{
		\resizebox{!}{0.8in}{\includegraphics{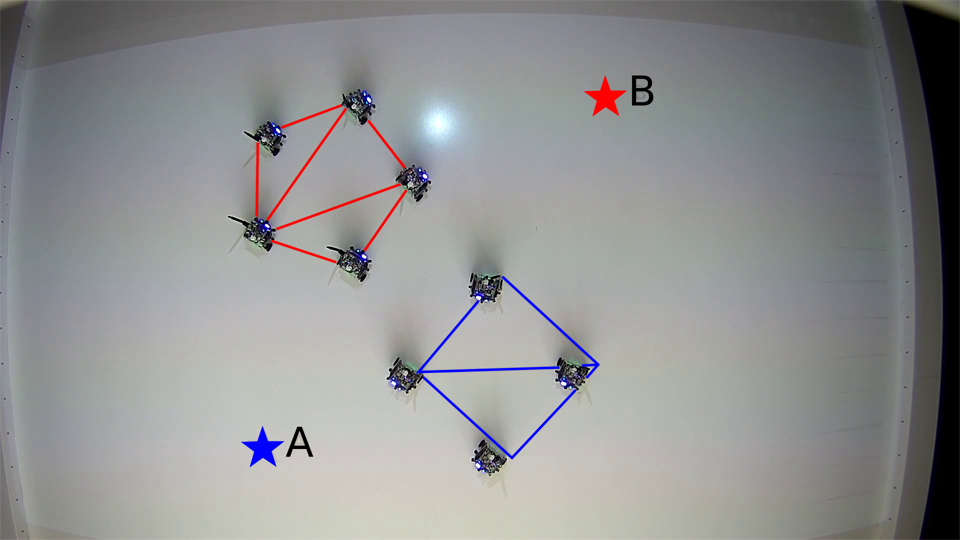}}
	}\quad
	\subfloat[$t=71\,\mathrm{s}$\label{fig:exp_cen_71}]{
    	\resizebox{!}{0.8in}{\includegraphics{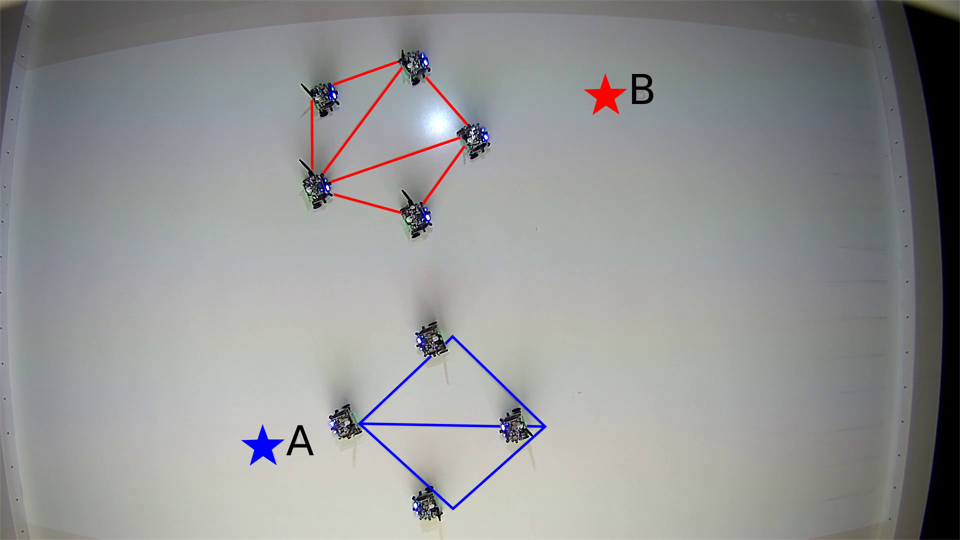}}
	}\quad
	\subfloat[$t=77\,\mathrm{s}$\label{fig:exp_cen_77}]{
		\resizebox{!}{0.8in}{\includegraphics{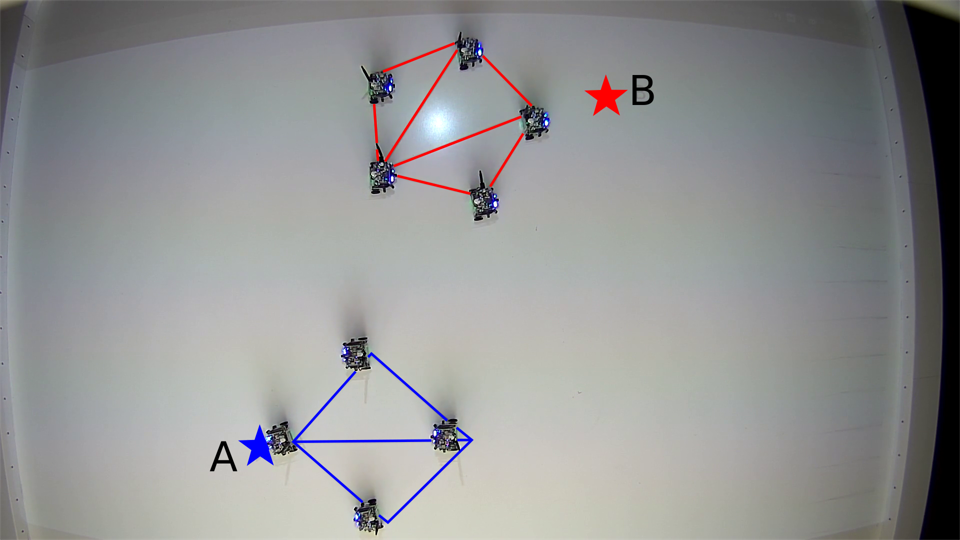}}
	}\vspace{-2mm}
    \caption{The snapshots from the formation preservation experiment with the centralized \flow framework. Goal positions and the formation graphs are projected onto the arena by an overhead projector. The colors of the graphics are augmented in the figures for the purpose of visualization. The two sub-teams of robots are tasked with maintaining the formation while moving back and force between two goal points in arena. The red and blue lines in the figure denote the formation graphs. The red and blue stars are the current goal positions for sub-team 1 and sub-team 2, respectively. \vspace{-7mm}
    }
    \label{fig:exp_cen}
\end{figure*}

\vspace{-4mm}
\subsubsection{Decentralized Partial \flow Framework}

For the decentralized partial \flow framework, we consider a team of eight robots. The robots are divided into two sub-teams. The task of the first sub-team is to achieve cyclic pursuit behavior for a circle of radius $1$ m centered at the origin. The other sub-team is designed to go through the circle surveilled by the other sub-team. To achieve the cyclic pursuit behavior, each robot in the first sub-team follows a point moving along the circle through a goal attractor RMP (defined in Appendix~\ref{app:attractor_rmp}).
The \forest for the second sub-team follows a similar structure as Fig.~\ref{fig:forest}. For each single robot, there are collision avoidance RMPs for all other robots. Snapshots from the experiment are shown in Fig.~\ref{fig:exp_dec}.  The robots from the second sub-team manage to pass through the circle under the decentralized  framework.

\begin{figure*}
    \centering
    \vspace{-7mm}
    \subfloat[$t=0\,\mathrm{s}$\label{fig:exp_dec_0}]{
		\resizebox{!}{0.8in}{\includegraphics{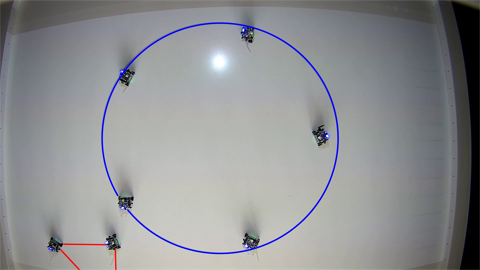}}
	}\quad
	\subfloat[$t=22\,\mathrm{s}$\label{fig:exp_dec_22}]{
    	\resizebox{!}{0.8in}{\includegraphics{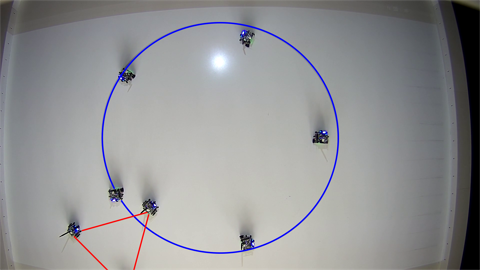}}
	}\quad
	\subfloat[$t=45\,\mathrm{s}$\label{fig:exp_dec_45}]{
		\resizebox{!}{0.8in}{\includegraphics{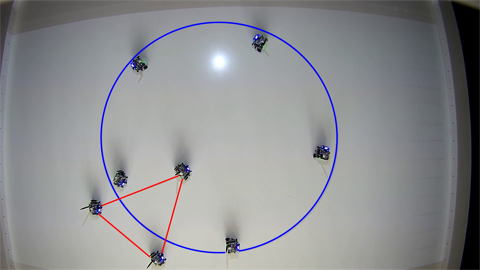}}
	}\vspace{-2mm}\\
	\subfloat[$t=95\,\mathrm{s}$\label{fig:exp_dec_95}]{
		\resizebox{!}{0.8in}{\includegraphics{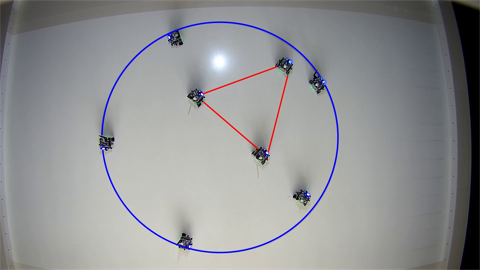}}
	}\quad
	\subfloat[$t=129\,\mathrm{s}$\label{fig:exp_dec_129}]{
    	\resizebox{!}{0.8in}{\includegraphics{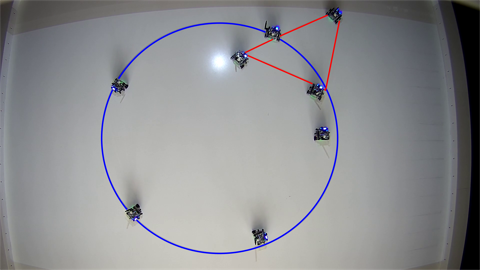}}
	}\quad
	\subfloat[$t=143\,\mathrm{s}$\label{fig:exp_dec_143}]{
		\resizebox{!}{0.8in}{\includegraphics{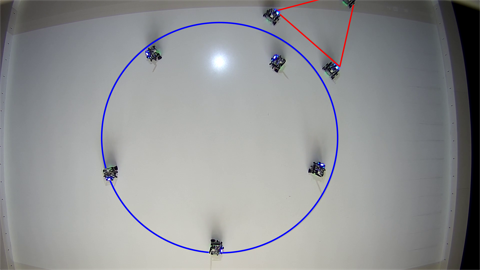}}
	} \vspace{-2mm}
    \caption{The snapshots from the experiment for the decentralized multi-robot RMP framework. The circle surveilled by the first sub-team and the formation graph for the second sub-team are projected onto the environment. Robots were divided into two sub-teams. The first sub-team of five robots performed a cyclic pursuit behavior for a circle of radius $1$ centered at the origin. The other sub-team passes through the circle surveilled by the other sub-team.\vspace{-5mm}
    }
    \label{fig:exp_dec}
\end{figure*}



\vspace{-5mm}
\section{Conclusions}
\label{sec:conclusion}
\vspace{-2mm}

In this paper, we consider multi-objective tasks for multi-robot systems. We argue that it is advantageous to define controllers for single subtasks on their corresponding manifolds. We propose centralized and decentralized algorithms to generate control policies for multi-robot systems by combining control policies defined for individual subtasks. The multi-robot system is proved to be stable under the generated control policies. We show that many existing potential-based multi-robot controllers can also be approximated by the proposed algorithms. Several subtask policies are proposed for multi-robot systems. The proposed algorithms are tested through simulation and deployment on real robots.

\begin{acknowledgment}
    This work was supported in part by the grant ARL DCIST CRA W911NF-17-2-0181.
\end{acknowledgment}

\vspace{-5mm}
\bibliographystyle{unsrt}
\bibliography{references}

\clearpage
\appendix
\section*{Appendices}
\vspace{-2mm}
\section{Proof of Theorem \ref{thm:stability_dec}}\label{app:proof_stability}
\stabilityDec*
\begin{proof}
Let $\Phi$ be the total potential function for all subtasks, i.e., $\Phi=\sum_{k\in\KK}\Phi_{\ltt_k}\circ\psi_{\ett_k}$. Consider the Lyapunov function candidate $V=\left(\sum_{i=1}^N K_i\right)+\Phi$, where $K_i=\frac{1}{2}\qd_i^\t\G_{\rtt}^i\qd_i.$
Then, following a derivation similar to~\cite{cheng2018rmpflow}, we have, \vspace{-2mm}
\begin{equation}
    \begin{split}
        \frac{d}{dt} K_i &=\frac{d}{dt}\Big(\frac{1}{2}\qd_i^\t\big(\sum_{k\in\KK_i} (\J_{\ett_k}^i)^\t \G_{\ltt_k}^i\J_{\ett_k}^i\big)\qd_i\Big)
        =\frac{d}{dt}\Big(\frac{1}{2}\sum_{k\in\KK_i} (\zd_{\ltt_k}^i)^\t \G_{\ltt_k}^i\zd_{\ltt_k}^i\Big)\\
        &=\sum_{k\in\KK_i}(\zd_{\ltt_k}^i)^\t \left(\G_{\ltt_k}^i\zdd_{\ltt_k}^i
        +\frac12\,\left(\frac{d}{dt}\G_{\ltt_k}^i\right)\zd_{\ltt_k}^i\right)\\
        &=\sum_{k\in\KK_i}(\zd_{\ltt_k}^i)^\t\Big(\M_{\ltt_k}^i\zdd_{\ltt_k}^i+\frac12\,\sdot{\Gb_{\ltt_k}^i}{\z_{k}^i}\zd_{\ltt_k}^i\Big)\\
        &=\sum_{k\in\KK_i}(\zd_{\ltt_k}^i)^\t\Big(\M_{\ltt_k}^i\J_{\ett_k}^i\qdd_i+\M_{\ltt_k}^i\Jd_{\ett_k}^i\qd_i+\frac12\,\sdot{\Gb_{\ltt_k}^i}{\z_{k}^i}\zd_{\ltt_k}^i\Big)\\
        &=\qd_i^\t\sum_{k\in\KK_i}(\J_{\ett_k}^i)^\t\M_{\ltt_k}^i\J_{\ett_k}^i\qdd_i+\sum_{k\in\KK_i}(\zd_{\ltt_k}^i)^\t\Big(\M_{\ltt_k}^i\Jd_{\ett_k}^i\qd_i+\frac12\,\sdot{\Gb_{\ltt_k}^i}{\z_{k}^i}\zd_{\ltt_k}^i\Big).
    \end{split}
\end{equation}

By definition of the \pullback operator, we have $\M_{\rtt}^i = \sum_{k\in\KK_i} (\J_{\ett_k}^i)^\t \M_{\ltt_k}^i\J_{\ett_k}^i$, \vspace{-2mm}
\begin{equation}
    \begin{split}
        \frac{d}{dt}K_i&=\qd_i^\t\M_{\rtt}^i\qdd_i+\sum_{k\in\KK_i}(\zd_{\ltt_k}^i)^\t\Big(\M_{\ltt_k}^i\Jd_{\ett_k}^i\qd_i+\frac12\,\sdot{\Gb_{\ltt_k}^i}{\z_{k}^i}\zd_{\ltt_k}^i\Big)\\
        &=\qd_i^\t\f_{\rtt}^i+\sum_{k\in\KK_i}(\zd_{\ltt_k}^i)^\t\Big(\M_{\ltt_k}^i\Jd_{\ett_k}^i\qd_i+\frac12\,\sdot{\Gb_{\ltt_k}^i}{\z_{k}^i}\zd_{\ltt_k}^i\Big).
    \end{split}
\end{equation}

Also by definition of \pullback, $\f_{\rtt}^i= \sum_{k\in\KK_i} (\J_{\ett_k}^i)^\t (\f_{\ltt_k}^i - \M_{\ltt_k}^i \dot{\J}_{\ett_k}^i \qd_i)$, hence,
\begingroup
\allowdisplaybreaks
\begin{align}
    \frac{d}{dt} K_i
    &= \sum_{k\in\KK_i}(\zd_{\ltt_k}^i)^\t\Big(\fb_{\ltt_k}^i-\M_{\ltt_k}^i\Jd_{\ett_k}^i\qd_i+\M_{\ltt_k}^i\Jd_{\ett_k}^i\qd_i+\frac12\,\sdot{\Gb_{\ltt_k}^i}{\z_{k}^i}\z_{\ltt_k}^i\Big)\nonumber\\
    &= \sum_{k\in\KK_i}(\zd_{\ltt_k}^i)^\t\Big(\fb_{\ltt_k}^i+\frac12\,\sdot{\Gb_{\ltt_k}^i}{\z_{k}^i}\z_{\ltt_k}^i\Big)\nonumber\\
    &=\sum_{k\in\KK_i}(\zd_{\ltt_k}^i)^\t\Big(- \nabla_{\z_k^i} \Phi_{\ltt_k}^i - \Bb_{\ltt_k}^i\,\zd_{k}^i - \frac12\,\sdot{\Gb_{\ltt_k}^i}{\z_{k}^i}\,\zd_{k}^i+\frac12\,\sdot{\Gb_{\ltt_k}^i}{\z_{k}^i}\z_{\ltt_k}^i\Big)\nonumber\\
    &=\sum_{k\in\KK_i}(\zd_{\ltt_k}^i)^\t\Big(- \nabla_{\z_k^i} \Phi_{\ltt_k}^i - \Bb_{\ltt_k}^i\,\zd_{k}^i\Big)\\
    &= -\Big(\sum_{k\in\KK_i}(\zd_{\ltt_k}^i)^\t\nabla_{\z_k^i} \Phi_{\ltt_k}^i\Big) - \qd_i^\t\Big(\sum_{k\in\KK_i}(\J_{\ett_k}^i)^\t\Bb_{\ltt_k}^i\,\J_{\ett_k}^i\Big)\qd_i\nonumber\\
    &= -\Big(\sum_{k\in\KK_i}(\zd_{\ltt_k}^i)^\t\nabla_{\z_k^i} \Phi_{\ltt_k}^i\Big) - \qd_i^\t\Bb_{\rtt}^i\qd_i\nonumber\\
    &= -\Big(\sum_{k\in\KK_i}(\zd_{\ltt_k}^i)^\t\nabla_{\z_k} \Phi_{\ltt_k}\Big) - \qd_i^\t\Bb_{\rtt}^i\qd_i,\nonumber
\end{align}%
\endgroup

where we denote $\z_k=\psi_{\ltt_k}(\q)$ and the last equation follows from the fact that $\z_{k}^i=\psi_{\ltt_k}=\z_k$ for all $i$.

Therefore, for the Lyapunov function candidate $V$, we have, \vspace{-1mm}
\begin{equation}
    \begin{split}
        \frac{d}{dt} V &= \sum_{i=1}^N \frac{d}{dt} K_i + \sum_{k\in\KK} \zd_{k}^\t\nabla_{\z_k}\Phi_{\ltt_k}\\
        &= -\sum_{i=1}^N\Big(\sum_{k\in\KK_i}(\zd_{k}^i)^\t\nabla_{\z_{k}}\Phi_{\ltt_k}\Big) +\sum_{k\in\KK}\zd_k^\t\nabla_{\z_{k}}\Phi_{\ltt_k}  -\sum_{i=1}^N \qd_i^\t\Bb_{\rtt}^i\qd_i\\
        &= -\sum_{k\in\KK}\Big(\sum_{i\in\II_k}(\zd_{k}^i)^\t\nabla_{\z_{k}}\Phi_{\ltt_k}\Big) +\sum_{k\in\KK}\zd_k^\t\nabla_{\z_{k}}\Phi_{\ltt_k}  -\sum_{i=1}^N \qd_i^\t\Bb_{\rtt}^i\qd_i\\
        &= -\sum_{i=1}^N\qd_i^\t\Bb_{\rtt}^i\qd_i,
    \end{split}
\end{equation}
where the last equation follows from $\zd_{k}=\sum_{i\in\II_l}\J_{\ett_k}^i\,\qd_i=\sum_{i\in\II_k}\zd_{k}^i$. Then by LaSalle's invariance principle~\cite{khalil1996noninear}, the system converges to a forward invariant set $\CC_\infty \coloneqq \{(\q,\qd) : \nabla_{\q_i} \Phi_{\rtt}^i = 0, \qd_i = 0,\forall i \in\II\}$.
\end{proof}

\vspace{-4mm}
\section{Details of the Experiments}\label{app:experiments}
\vspace{-2mm}
In this appendix, we introduce the construction of unitary goal attractor RMP, which is used in many of the experiments, and provide the choice of parameters for the simulation and experiments.

\vspace{-2mm}
\subsection{Unitary Goal Attractor RMP}\label{app:attractor_rmp}
In multi-robot scenarios, instead of planning paths for \emph{every} robot,
it is common to plan a path or assign a goal to one robot, called the \emph{leader}. The other robots may simply follow the leader or maintain a given formation depending on other subtasks assigned to the team. 
In this case, a goal attractor RMP may be assigned to the leader. A number of controllers for multi-robot systems are also based on going to a 
goal position, such as the cyclic pursuit behavior \cite{cortes2017coordinated} and Voronoi-based coverage controls \cite{cortes2017coordinated,cortes2004coverage}.

There are several design options for goal attractor RMPs. We will discuss two examples. The first goal attractor RMP is introduced in \cite{cheng2018rmpflow}. The attractor RMP for robot $i$ is defined on the subtask space $\z = \x_i - \g_i$, where $\g_i$ is the desired configuration for the robot. The metric is designed as $\G(\z) = w(\z)\,\I$. The weight function $w(\z)$ is defined as $w(\z) = \gamma(\z)\,w_u + (1-\gamma(\z))\,w_l$, with $0\leq w_l\leq w_u<\infty$ and $\gamma(\z) = \exp(-\frac{\|\z\|^2}{2\sigma^2})$ for some $\sigma>0$. The weights $w_l$ and $w_u$ control the importance of the RMP when the robots are far from the goal and close to the goal, respectively.  As the robot approaches the goal, the weight $w(\z)$ will smoothly increase from $w_l$ to $w_u$. The parameter $\sigma$ determines the characteristic length of the metric. The main intuition for the metric is that when the robot is far from the goal, the attractor should be permissive enough for other subtasks such as collision avoidance, distance preservation, etc. However, when the robot is close to the goal, the attractor should have high importance so that the robot can reach the goal. \mbox{The potential function is designed such that, }\vspace{-1mm}
\begin{equation} \label{eqn:sNumerical}
\nabla_\z\Phi(\z) = \beta\,w(\z)\left(
\frac{1 - e^{-2\alpha\|\z\|}}{1 + e^{-2\alpha\|\z\|}}
\right) \hat{\z}
= \beta\,w(\z)\,s_\alpha\big(\|\z\|\big)\,\hat{\z},
\end{equation}
where $\beta>0$, $s_\alpha(0) = 0$ and $s_\alpha(r)\rightarrow 1$ as $r\rightarrow\infty$. The parameter
$\alpha$ determines the characteristic length of the potential. The potential function defined in (\ref{eqn:sNumerical}) provides a \emph{soft-normalization} for $x$ so that the transition near the origin is smooth. The damping matrix is $\Bb(\z)=\eta\,w(\z)\,\I$, where $\eta>0$ is a positive scalar.
We will refer to this goal attractor RMP as \emph{Goal Attractor RMPa} in subsequent sections. Although more complicated, it produces better results when combined with other RMPs, especially \mbox{collision avoidance RMPs (see Appendix \ref{app:simulation}).}

Another possible goal attractor RMP is based on a PD controller. This RMP is also defined on the subtask space $\z=\psi(\x_i)=\x_i-\g_i$. The metric is a constant times identity matrix, $\G = c\,\I$ with some $c>0$. The potential function is defined as $\Phi(\z)=\frac{1}{2} \alpha\|\z\|^2$ and the damping is $\B\equiv\eta\,\I$, with $\alpha, \eta >0$. This RMP is equivalent to a PD controller with $k_p=\alpha/c$ and $k_d=\eta/c$. This goal attractor will be referred to as \emph{Goal Attractor RMPb} in subsequent sections.

\vspace{-2mm}
\subsection{Choice of Parameters}
\subsubsection{Simulated Formation Preservation Task}
For each robot, we define a goal attractor RMPa, with parameters $w_u=10$, $w_l=1$, $\sigma=0.1$, $\beta=0.1$,  $\alpha=10$, $\eta=1$. We use a damper RMP with $\G\equiv0.01\,\I$ $\B\equiv \I$, $\Phi\equiv 0$. For the distance preserving RMPa we set parameters $\G=c=1$ and $\eta=2$. For distance preservation RMPb (which is shown to be equivalent to the potential-based controller in the previous simulation), we choose parameters $\G=\I$ and $\eta=2$.

\vspace{-2mm}
\subsubsection{Centralized \flow Framework} We use distance preservation RMPa's with $\G=c=10$, and $\eta=5$. The safety distance between robots is $d_S=0.18$. The parameters for the collision avoidance RMPs are set as $\alpha=1{\mathrm{e}}-5$, $\epsilon=1{\mathrm{e}}-8$, and $\eta=0.5$. For goal attractors, we use goal attractor RMPa's with $w_u = 10, w_l = 0.01, \sigma = 0.1, \beta = 1, \alpha = 1$, and $\eta = 1$. The damping RMPs have parameters $\G\equiv0.01\,\I$ $\B\equiv \I$, $\Phi\equiv 0$.

\vspace{-2mm}
\subsubsection{Decentralized Partial \flow Framework}
For the cyclic pursuit tasks, the robots are attracted by points moving along the circle with angular velocity $0.06\,\mathrm{rad/s}$. The parameters for the associated goal attractor RMPa's are $w_u=10$, $w_l=0.01$, $\sigma=0.1$, $\beta=1$,  $\alpha=1$, $\eta=1$. For robots from sub-team 2, the distance preservation RMPa's have parameters $\G=c=10$, $\eta=2$. The goal attractor are goal attractor RMPa's with parameters $w_u=10$, $w_l=1$, $\sigma=0.1$, $\beta=1$,  $\alpha=10$, $\eta=2$. The parameters for the collision avoidance RMPs are $\alpha=1{\mathrm{e}}-5$, $\epsilon=1{\mathrm{e}}-8$, and $\eta=1$, with safety distance $d_S=0.18$.

\vspace{-2mm}
\section{Additional Simulation Results}\label{app:simulation}
\vspace{-2mm}

\subsubsection{RMPs \& Potential-based Controllers}
As is discussed in Section \ref{sec:centralized}, many potential-based multi-robot controllers can be reconstructed by the RMP framework up to degree normalization. In this example, we consider a formation control task with five robots. The robots are tasked with forming a regular pentagon with circumcircle radius $0.4$. The robots are initialized with a regular pentagon formation, but with a larger circumcircle radius of $1$.
\begin{figure}
    \centering
    \vspace{-4mm}
    \subfloat[Potential-based\label{fig:rmp_pf_pf}]{
		\resizebox{!}{1.2in}{\includegraphics{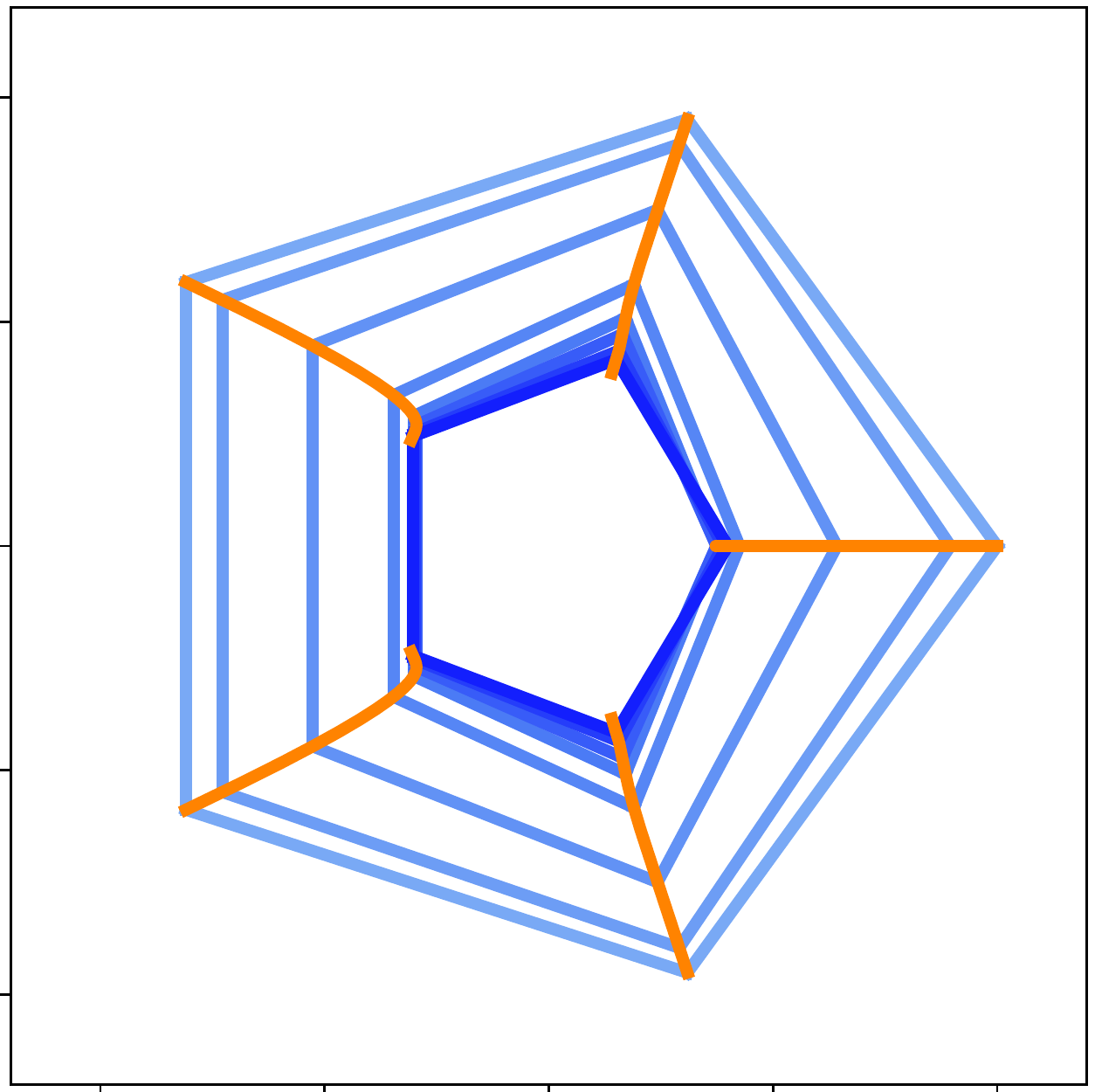}}
	}
	\subfloat[Dist. Presv. RMPb\label{fig:rmp_pf_rmp1}]{
    	\resizebox{!}{1.2in}{\includegraphics{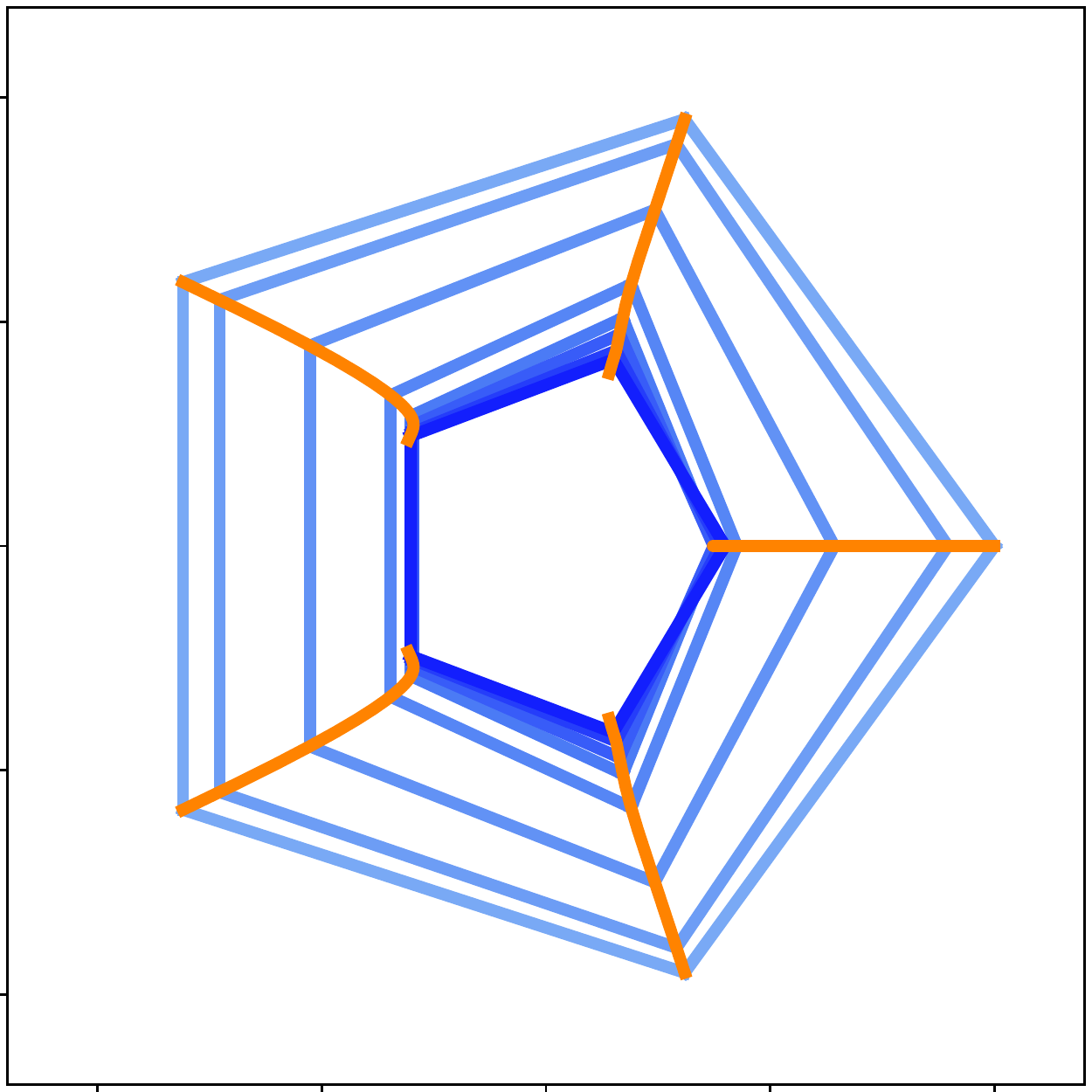}}
	}
	\subfloat{
		\resizebox{!}{1.2in}{\includegraphics{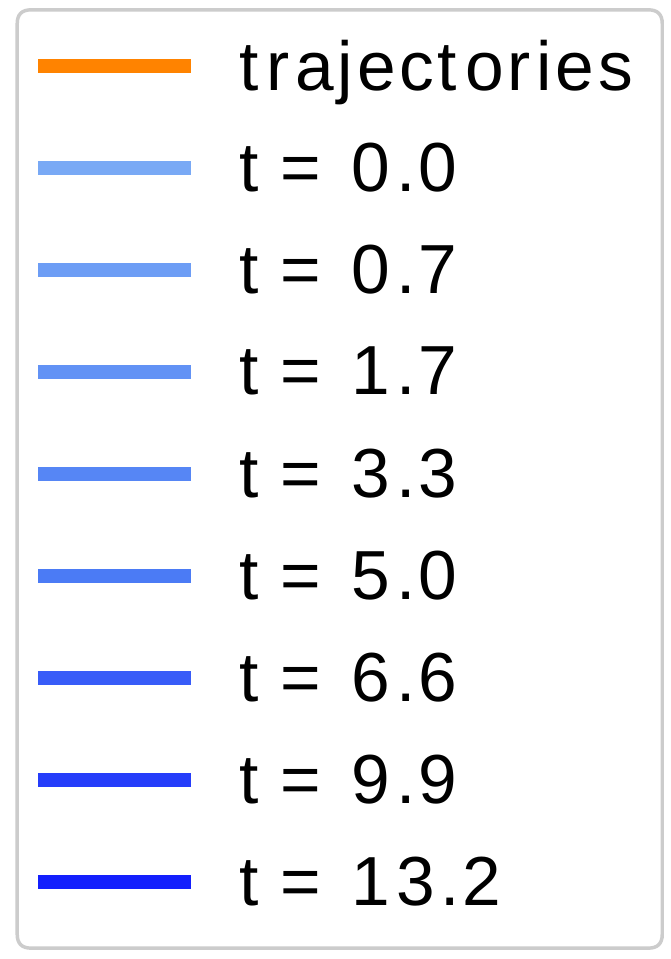}}
	}
	\caption{(a) The potential-based controller introduced in~\cite{mesbahi2010graph}. The blue pentagons from light to dark denotes the shape from $t=0.0s$ to $t=13.2s$. The orange curves represent the trajectories of the robots. (b) The controller generated from the centralized RMP framework with Distance Preservation RMPa. The fact that (a) and (b) are identical shows that the two controllers have the same behavior. \vspace{-8mm}
	}
    \label{fig:rmp_pf}
\end{figure}

We consider a degree-normalized potential field controller from~\cite{mesbahi2010graph}, \vspace{-1mm}
\begin{equation}\label{eq:formation_pf}
    \begin{split}
        \ub_i &= -\frac{1}{D_i} \sum_{j:(i,j)\in E}\Big(\nabla_{\x_i}\big\{\frac{1}{2}(\|\x_i-\x_j\|-d_{ij})^2\big\} - \eta\,\xd_i\Big)\\
        &= -\frac{1}{D_i} \sum_{j:(i,j)\in E}\Big(\frac{\|\x_i-\x_j\|-d_{ij}}{\|\x_i-\x_j\|}(\x_i-\x_j) - \eta\,\xd_i\Big),
    \end{split}
\end{equation}
where $d_{ij}$ is the desired distance between robot $i$ and robot $j$, $E$ is the set of edges in the formation graph, and $D_i$ is the degree of robot $i$ in the formation graph. For the RMP implementation, we use the controller given by (\ref{eq:formation_b}).
The potential-based controller (\ref{eq:formation_pf}) is equivalent to the controller generated by the distance preservation RMP given by (\ref{eq:formation_b}) when choosing $c=\alpha=1$. In simulation, we choose $\eta=2$ for both the RMP controller and the potential-based controller.
The trajectories of the robots under the two controllers are displayed in Fig.~\ref{fig:rmp_pf_pf} and Fig.~\ref{fig:rmp_pf_rmp1}, respectively. The results are \emph{identical} implying  the controllers have exactly the same behavior.

\vspace{-2mm}
\subsubsection{Goal Attractor RMPs \& Collision Avoidance RMPs}
An advantage of the multi-robot RMP framework is that it can leverage  existing single-robot controllers, which may have desirable properties, especially when combined with other controllers. In this simulation task, the performance of the two goal attractor RMPs are compared when combined with pairwise collision avoidance RMPs. In the simulation, three robots are tasked with reaching a goal on the other side of the field while avoiding collisions with each of the others. The parameters for the collision avoidance RMPs are set as $\alpha=\epsilon=1{\mathrm{e}}-5$, and $\eta=0.2$.
Fig.~\ref{fig:g2g_rmp1} and Fig.~\ref{fig:g2g_rmp2} show the behavior of the resulting controllers with the two choices goal attractor RMPs discussed in Section~\ref{app:attractor_rmp}, respectively. For goal attractor a, we use parameters $w_u = 10, w_l = 0.01, \sigma = 0.1, \alpha = 1$, and $\eta = 1$. we use $c=1$, $\alpha=1$, and $\eta=2$ for goal attractor RMPb.
We notice that goal attractor RMPa generates smoother trajectories compared to goal attractor RMPb. 

\vspace{-2mm}
\subsubsection{The Centralized \& Decentralized RMP Algorithms}
The centralized and the decentralized RMP algorithms are also compared through the same simulation of three robots reaching goals. Goal attractor RMPa's with the same parameters were used. For the collision avoidance RMP, we set $\alpha=\epsilon=1{\mathrm{e}}-5$, and $\eta=0.2$.
The trajectories of the robots under the decentralized algorithm are illustrated in Fig.~\ref{fig:g2g_rmp3}. Compared to trajectories generated from centralized RMPs, the robots oscillate slightly when approaching other robots, and made aggressive turns to avoid collisions.

\begin{figure}
    \centering
    \vspace{-6mm}
    \subfloat[Centralized/RMPa\label{fig:g2g_rmp1}]{
		\resizebox{!}{1.3in}{\includegraphics{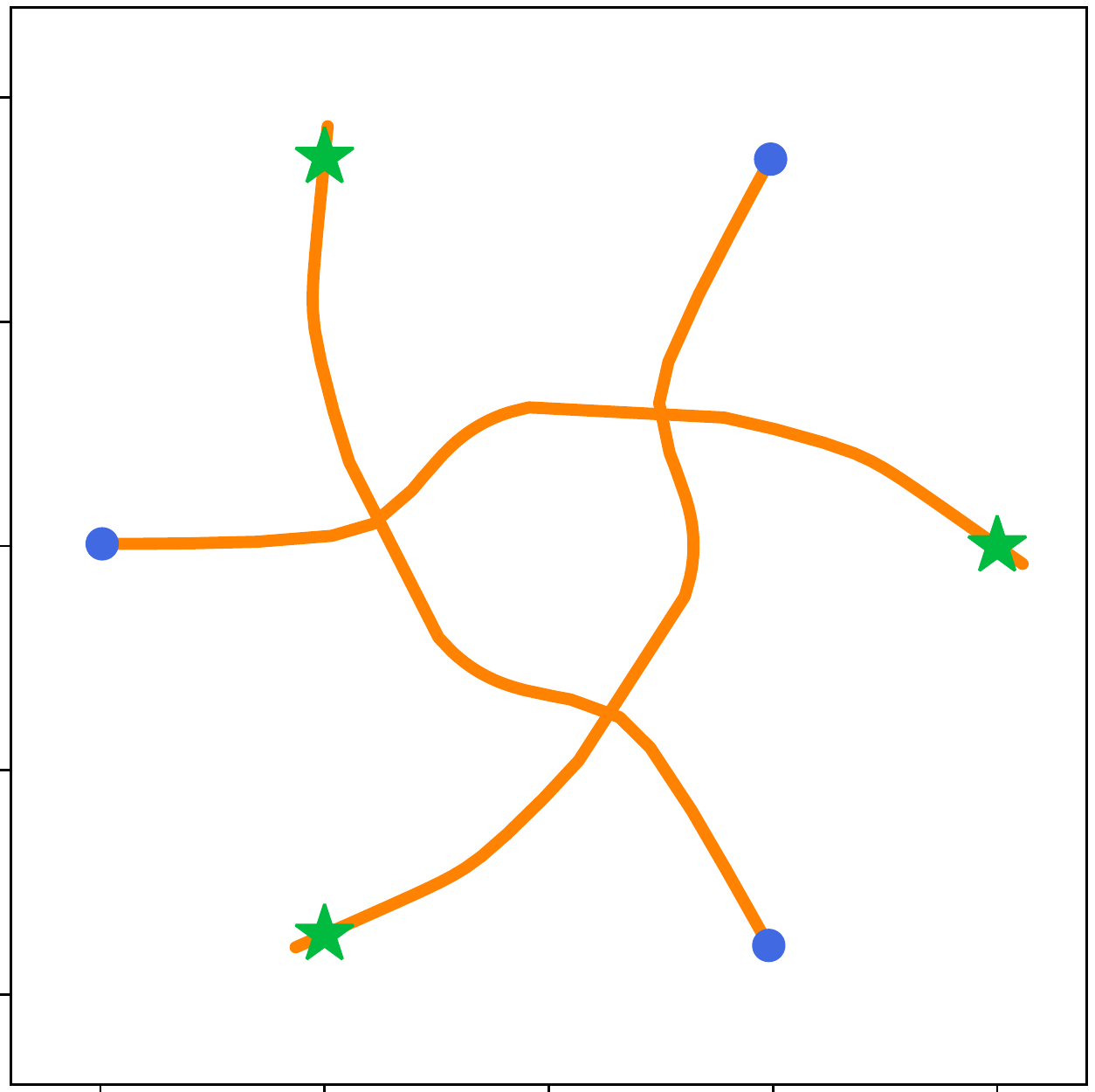}}
	}
	\,
	\subfloat[Centralized/RMPb\label{fig:g2g_rmp2}]{
    	\resizebox{!}{1.3in}{\includegraphics{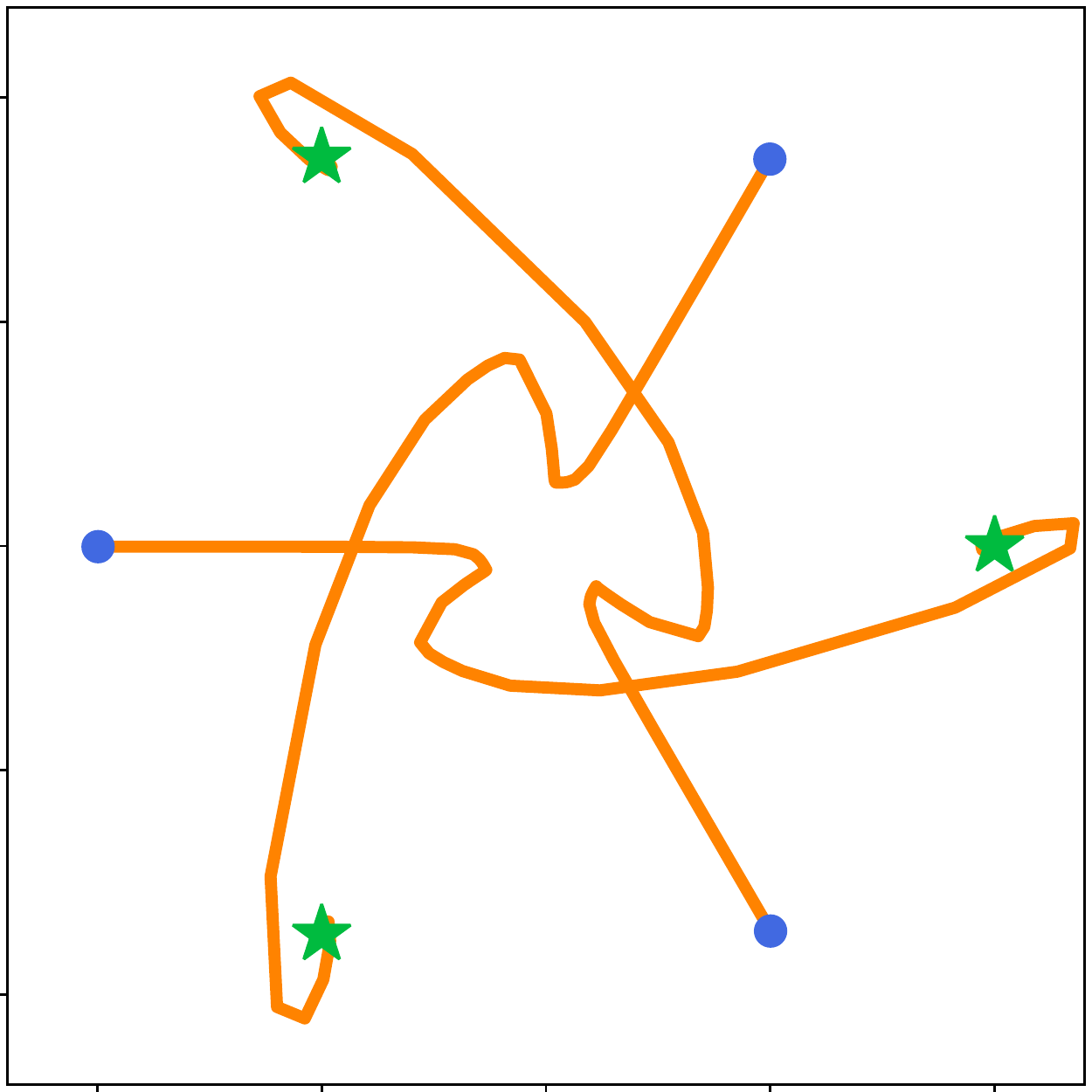}}
	}
	\,
	\subfloat[Decentralized/RMPa\label{fig:g2g_rmp3}]{
    	\resizebox{!}{1.3in}{\includegraphics{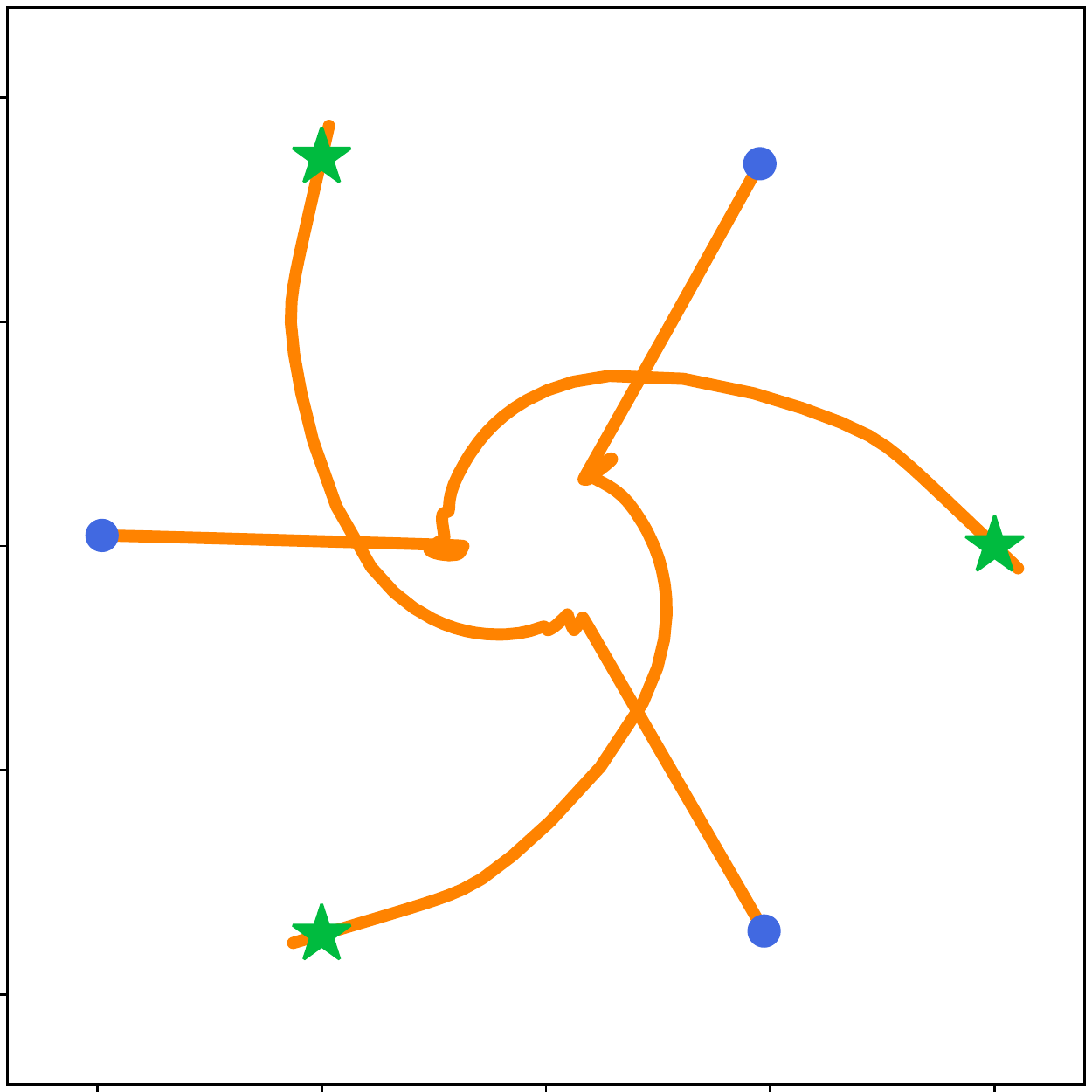}}
	}
    \caption{The performance of goal attractor RMPs combined with pairwise collision avoidance RMPs. The blue dots and green stars denote the initial and goal positions, respectively. The trajectories are represented by orange curves. (a) The more sophisticated goal attractor RMPa generates smooth trajectories when combined with collision avoidance RMP in the centralized framework. (b) Goal attractor RMPb, which is equivalent to a PD controller, generates more distorted trajectories. (c) Under the decentralized RMP framework, the robots oscillate slightly near the origin and turn abruptly against each other. }\vspace{-6mm}
    \label{fig:g2g}
\end{figure}

\section{On Heterogeneous Robotic Teams}\label{app:heterogeneity}
\vspace{-2mm}

A significant feature of RMPs is that they are intrinsically coordinate-free~\cite{cheng2018rmpflow}. Consider two robots $i$ and $j$ with configuration space $\CC_i$ and $\CC_j$, respectively. Assume that there exists a smooth map $\psi$ from $\CC_j$ to $\CC_i$. Then the \tree designed for one robot $i$ can be directly transferred to robot $j$ by connecting the tree to the root node of robot $j$ through the map $\psi$. Therefore, RMPs provides a level of abstraction for heterogeneous robotic teams so that the user only needs to design desired behaviors for a homogeneous team with simple dynamics models, for example, double integrator dynamics, and seamlessly transfer it to the heterogeneous team. This insight could bridge the gap between theoretical results, which are usually derived for homogeneous robotic teams with simple dynamics models, and real robotics applications.

\end{document}